\documentclass[twoside,11pt]{article}

\newtheorem{theorem}{Theorem}

\newtheorem{lemma}{Lemma}

\newtheorem{assum}{Assumption}
\newtheorem{remark}{Remark}
\usepackage{jmlr2e}
\usepackage{epsfig}
\usepackage{graphicx}
\usepackage{algorithm}
\usepackage{algorithmic,amssymb}
\usepackage{amsmath}
\allowdisplaybreaks
\usepackage{natbib}
\usepackage{color}
\usepackage{times}
\usepackage[hyphens]{url}
\usepackage{soul}
\usepackage{url}
\usepackage{caption}
\usepackage{multirow}
\usepackage{makecell}
\usepackage[T1]{fontenc}
\usepackage[utf8]{inputenc}
\usepackage{placeins}
\usepackage{amsfonts,amstext,mathrsfs}
\usepackage{subcaption}

\usepackage{booktabs}       
\usepackage{nicefrac}       
\usepackage{microtype}      
\usepackage{multirow}
\usepackage{makecell}
\usepackage{bbding,threeparttable} 

\newtheorem{innercustomgeneric}{\customgenericname}
\providecommand{\customgenericname}{}

\newcommand{\ie}{\emph{i.e.}}

\newcommand{\BE}{\mathbb{E}}

\newcommand{\BW}{{\bf w}}

\newcommand{\BWW}{{\bf \widehat{w}}}

\newcommand{\whv}{{\widehat{v}}}
\newcommand{\bv}{{\bar{v}}}

\newcommand{\mt}{{m_t}}
\newcommand{\bo}{\mathcal{O}}

\usepackage{geometry}
\begin{document}

\title{Desirable Companion for Vertical Federated Learning: New Zeroth-Order Gradient Based Algorithm}
\author{\name Qingsong Zhang \email qszhang1995@gmail.com      \\
\addr Xidian University \& JD Tech.
\\
\name Bin Gu \email {jsgubin@gmail.com} \\
\addr MBZUAI \& JD Finance America Corporation
\\
\name Zhiyuan Dang \email {zhiyuandang@gmail.com} \\
\addr Xidian University \& JD Tech
\\
\name Cheng Deng	\email {chdeng.xd@gmail.com} \\
 \addr School of Electronic Engineering, Xidian University
 \\
 \name Heng Huang \email {heng.huang@pitt.edu} \\
\addr JD Finance America Corporation \& University of Pittsburgh\\
}
\editor{}

\maketitle

\setcounter{equation}{11}

	\begin{abstract}
		Vertical federated learning (VFL) attracts increasing attention due to the emerging demands of multi-party collaborative modeling and concerns of privacy leakage. A complete list of metrics to evaluate  VFL algorithms should include  model applicability, privacy security, communication cost, and computation efficiency, where privacy security is especially important to VFL.
		However, to the best of our knowledge, there does not exist a VFL algorithm  satisfying all these  criteria very well. To address this challenging problem, in this paper, we reveal that zeroth-order optimization (ZOO) is a desirable companion  for VFL.   Specifically, ZOO can 1) improve the model applicability of VFL framework, 2) prevent VFL framework from privacy leakage under curious, colluding, and malicious threat models,  3) support inexpensive communication and efficient computation. Based on that, we propose a novel and  practical VFL framework with black-box models, which is inseparably interconnected to the promising properties of  ZOO. We believe that it takes one stride towards designing a practical VFL framework matching all the criteria.
		Under this framework, we raise two novel {\bf asy}nchronous ze{\bf r}oth-ord{\bf e}r algorithms for {\bf v}ertical f{\bf e}derated {\bf l}earning (AsyREVEL) with different smoothing techniques. We theoretically drive the convergence rates of AsyREVEL algorithms under nonconvex condition. More importantly, we  prove the privacy security of our proposed framework under existing VFL attacks on different levels.  Extensive experiments on benchmark datasets demonstrate the favorable model applicability, satisfied privacy security, inexpensive communication, efficient computation, scalability and losslessness of our framework.
	\end{abstract}

\section{Introduction}\label{sec-introduction}
Federated learning \cite{mcmahan2016communication,smith2017federated,kairouz2019advances,gascon2016secure} is a prevailing distributed machine learning  paradigm for collaboratively training a machine learning model with privacy-preserving. A line of recent works \cite{mcmahan2016communication,smith2017federated} focus on the horizontal federated learning, where different parties have  different samples IDs but they all share the same complete features. The other line of  works \cite{hardy2017private,yang2019federated,zhang2021asysqn,zhang2021secure} studying the vertical federated learning (VFL), where data owned by different parties have the same sample IDs but disjoint subsets
of features.
Such scenario is common in the industry applications of emerging cross-organizational collaborative learning, including but not limited to medical study, financial risk, and  targeted marketing \cite{gong2016private,yang2019federated,cheng2019secureboost,hu2019fdml}.
For example, E-commerce companies owning the online shopping information could collaboratively train joint-models with banks and digital finance companies that own other information of the same people such as the average monthly deposit and online consumption, respectively, to achieve a precise customer profiling.
In this paper, we focus on VFL due to its wide applications for emerging multi-organization collaborative modeling with privacy-preserving.

There have been extensive works studying VFL systems from several perspectives. For example, some works focus on developing fast and scalable optimization methods for training VFL models based on stochastic gradient descent (SGD) methods \cite{wan2007privacy,hu2019fdml,liu2019communication,gu2020Privacy} and stochastic quasi-Newton methods \cite{yang2019quasi}. Some works study attack models under different threat models,
such as inference attack under the honest-but-curious \cite{gu2020Privacy}, inference attacks under the honest-but-colluding \cite{cheng2019secureboost,weng2020privacy} and backdoor attack under the malicious \cite{liu2020backdoor}. And some works study different (auxiliary) defense strategies such as the scalar product protocol for defense \cite{du2004privacy,hu2019fdml,liu2019communication,gu2020Privacy} and the auxiliary strategies, including the differential privacy (DP) \cite{liu2019boosting,xu2019hybridalpha,chen2020vafl} and the gradient sparsification \cite{liu2020backdoor}, for alleviating different attacks. Besides, there are several works focusing on reducing the communication cost (number of communication rounds) \cite{liu2019communication,yang2019quasi}, and also some works concerning different computation manners such as the synchronous \cite{gong2016private,zhang2018feature,liu2019communication} and  asynchronous ones \cite{hu2019fdml,gu2020Privacy}.

In fact, the above  perspectives can be summarized into a complete list of criteria, \ie, model applicability, privacy security, communication cost and computation efficiency, which can be used to comprehensively evaluate the performance of a VFL algorithm. Specifically,
1) {\bf model applicability} means the ability to solve different problems,	\cite{wan2007privacy,hu2019fdml,liu2019communication,gu2020Privacy,yang2019quasi},
2) {\bf privacy security} depends on the ability to defense different attacks, which is especially important to VFL \cite{gu2020Privacy,cheng2019secureboost,weng2020privacy,liu2020backdoor},
3) {\bf communication cost} depends one the number of communication rounds and the per-round communication overhead (PRCO) \cite{liu2019communication,yang2019quasi}, and
4) {\bf computation efficiency} is mainly  dominated by the computation manner, \ie, the asynchronous or synchronous \cite{hu2019fdml,gu2020Privacy,chen2020vafl,liu2019communication}.

However, to the best of our knowledge, there does not exist a VFL algorithm that is well designed to satisfy all these metrics together. Specifically,

\begin{enumerate}
	\item
	Most existing VFL frameworks adopt SGD methods \cite{wan2007privacy,hu2019fdml,liu2019communication,gu2020Privacy}. However, these optimization methods will fail when applied to the widely-existing problems whose explicit expressions of gradients are difficult or infeasible to obtain, such as the structure prediction \cite{sokolov2018sparse}, bandit learning \cite{shamir2017optimal} and black-box learning \cite{liu2018zeroth} problems.
	Thus, these VFL frameworks have the poor model applicability when applied to these problems.  	
	\item 	Privacy security is especially important for VFL, thus there have many attack manners \cite{gu2020Privacy,cheng2019secureboost,weng2020privacy,liu2020backdoor} and defense (or auxiliary defense) strategies \cite{liu2019boosting,xu2019hybridalpha,chen2020vafl,liu2020backdoor} been proposed. However, they still can not totally defense some existing VFL attacks, especially,  the latest proposed inference attacks in \cite{luo2020feature,weng2020privacy,liu2020backdoor} and backdoor attack in \cite{liu2020backdoor} due to transmitting the informative knowledge, \emph{e.g.}, the model parameters and (intermediate) gradients. Thus, existing VFL frameworks have the unsatisfied privacy  security.
	
	\item	Meanwhile, in the real VFL system, different parties always represent different companies or organizations across different networks.  In this case, most existing VFL frameworks directly transmitting the model parameters \cite{gong2016private,yang2019quasi,liu2019boosting,xu2019hybridalpha} or gradients \cite{weng2020privacy,chen2020vafl,liu2020backdoor} between parties are much communication-expensive due to the large PRCO.
	
	\item 	Moreover, it is common in the real-world applications that both large and small companies collaboratively learn the model, where the former have better computational capacity while the later have the poorer. In this case,  algorithms using synchronous computation \cite{gong2016private,zhang2018feature} are inefficient. Because, parties possessing better computational capacity have to waste the computational capacity to wait the stragglers for synchronization.
\end{enumerate}	
As discussed above, although there have been extensive works towards studying better VFL frameworks following these criteria, existing VFL frameworks still can not satisfy all criteria well because of the poor model applicability, the unsatisfied privacy security, expensive communication, or the inefficient computation. Thus, it is challenging to design a practical VFL framework that not only supports inexpensive communication and efficient computation but also has favorable model applicability and satisfied privacy security.

In this paper, we address this challenging problem by  revealing the promising properties of ZOO and, be inseparably interconnected, proposing a novel practical VFL framework with black-box models, under which the asynchronous zeroth-order optimization algorithms (AsyREVEL) are proposed.  Specifically, 1) ZOO only needs the function values for updating  rather than the gradients with explicit expressions and  thus can improve the model applicability of VFL to more ML problems. 2) Only black-box information \ie, function values, is necessary to be transmitted for ZOO, which can prevent existing VFL attacks under three levels of threat models, \ie, the curious, colluding, and malicious. 3) Only function values are transmitted for ZOO (have low PRCO) and the asynchronous computation is adopted for AsyREVEL, thus, ZOO-VFL can also support inexpensive communication and efficient computation.  We summarize the contributions of this paper as follows.
\begin{itemize}
	\item
	We are the first to reveal that ZOO is a desirable companion for VFL,  which not only support inexpensive communication and efficient computation but also has favorable model applicability and satisfied privacy security. Moreover, we also propose a novel practical VFL framework with black-box models, which inherits the promising properties of ZOO.
	\item
	We propose two AsyREVEL algorithms with different smoothing techniques, {\ie} AsyREVEL-Gau and -Uni, under our practical VFL framework. Moreover, we theoretically prove their convergence rates for the nonconvex problems.
\end{itemize}

\begin{table*}[!t]
	\centering
	\caption{A summary of evaluating existing VFL frameworks following these four metrics, where ERCR denotes ``exchanging the raw computation results'', TIG denotes ``transmitting intermediate gradients'', TG denotes ``transmitting  gradients'', MA denotes ``Model Aapplicability'', PS denotes ``Privacy Security'', IC denotes ``Inexpensive Communication'', CE denotes ``Computational Efficiency'',  VFL framework adopting AsyREVEL is proposed in Section \ref{sec-framework}, and the results of privacy security means which attack these methods cannot defense (``1--feature inference attack \cite{gu2020Privacy}'', ``2--label inference attack in \cite{liu2020backdoor}'', ``3--feature inference attack in \cite{luo2020feature}'', ``4--reverse multiplication and reverse sum attacks in \cite{weng2020privacy}'', ``5--backdoor attack in \cite{liu2020backdoor}'', ``--'' means that can prevent attacks  proposed in \cite{gu2020Privacy,luo2020feature,weng2020privacy,liu2020backdoor}). }
	\label{tab:comparsion}
	{\small{\begin{threeparttable}
				\begin{tabular}{@{}ccccc@{}}
					\toprule
					Methods                                                    & MA & PS & IC  & CE \\ \midrule
					Asynchronous ERCR-based methods \cite{hu2019fdml,gu2020Privacy}
					& \XSolidBrush                  & 2 \tnote{1}        & \Checkmark                        & \Checkmark                        \\
					Communication-efficient TIG-based method \cite{liu2019communication}
					& \XSolidBrush                   & 2,5        & \Checkmark                        & \XSolidBrush                        \\
					Asynchronous TG-based methods \cite{vepakomma2018split,chen2020vafl}
					& \XSolidBrush                   & 3       & \XSolidBrush                        & \Checkmark                        \\
					Communication-efficient HE-based method \cite{yang2019quasi}
					& \XSolidBrush                   & 4        & \Checkmark                        & \XSolidBrush                        \\
					Synchronous HE-based methods \cite{gong2016private,hardy2017private}
					& \XSolidBrush                   & 4        & \XSolidBrush                        & \XSolidBrush                        \\
					{\bf VFL framework adopting AsyREVEL (ours)}
					& \Checkmark                   & {\bf --}       & \Checkmark                        & \Checkmark                        \\
					\bottomrule
				\end{tabular}
				\begin{tablenotes}
					\item [1] When not all parties have the labels, these methods can not prevent the label inference attack \cite{liu2020backdoor}.
				\end{tablenotes}
	\end{threeparttable}}}
\end{table*}

\begin{figure*}[!t]
	\centering
	\begin{subfigure}{0.3\linewidth}
		\includegraphics[width=\linewidth]{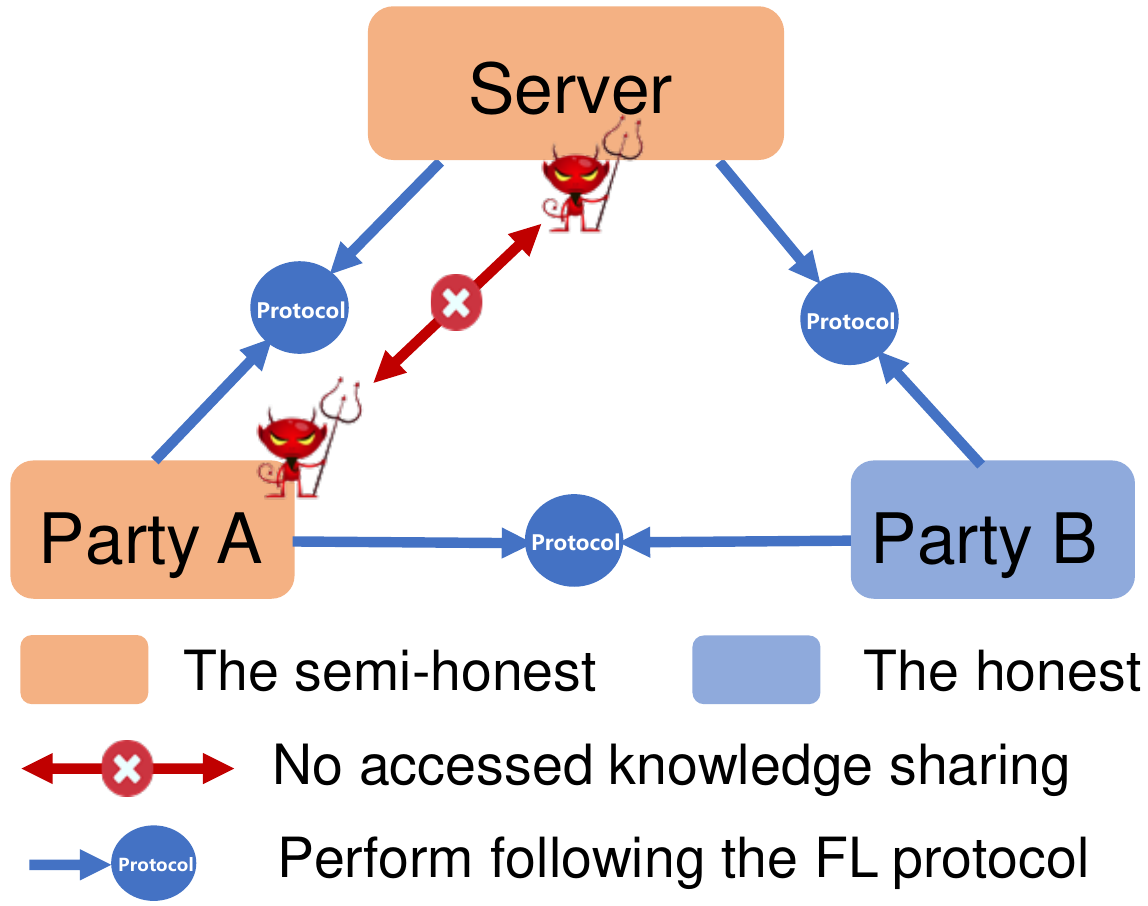}
		\caption{Semi-honest with curiosity}
	\end{subfigure}
	\qquad
	\begin{subfigure}{0.3\linewidth}
		\includegraphics[width=\linewidth]{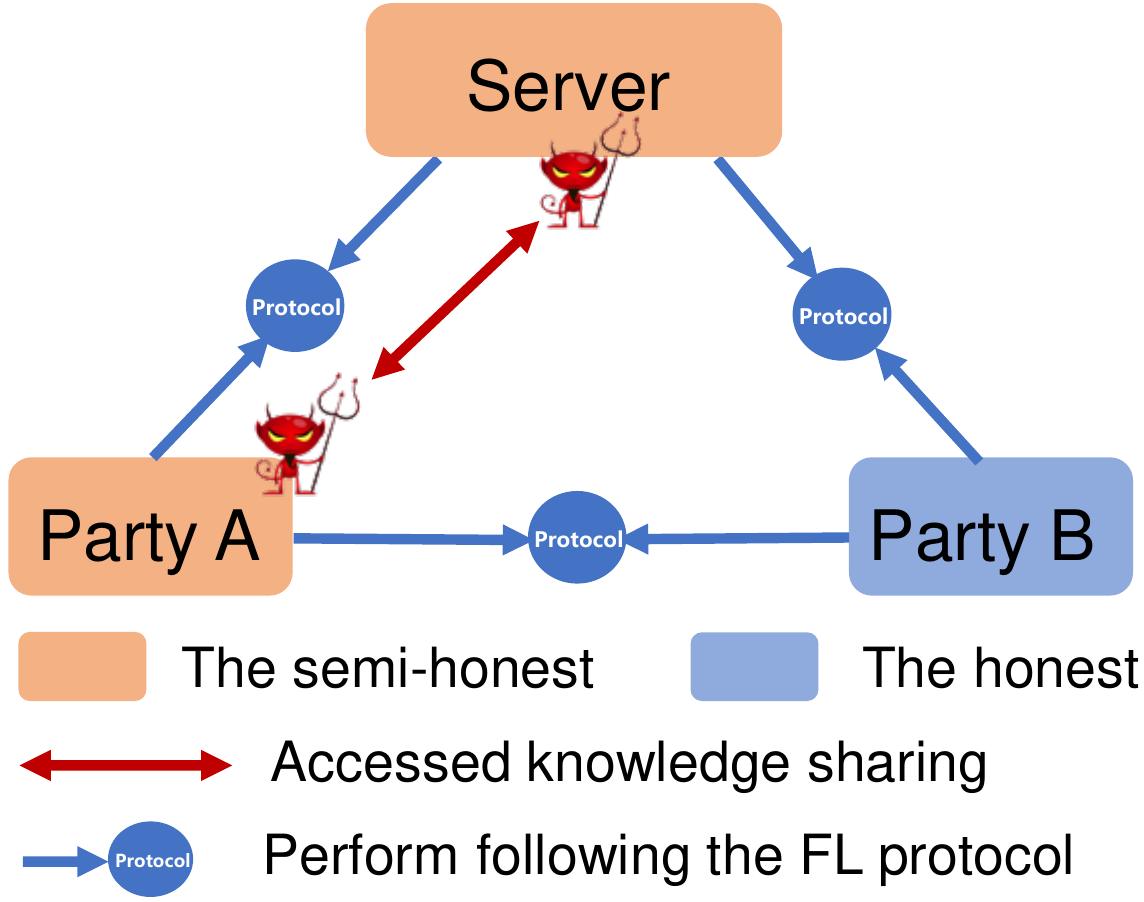}
		\caption{Semi-honest with collusion}
	\end{subfigure}
	\begin{subfigure}{0.3\linewidth}
		\qquad
		\includegraphics[width=\linewidth]{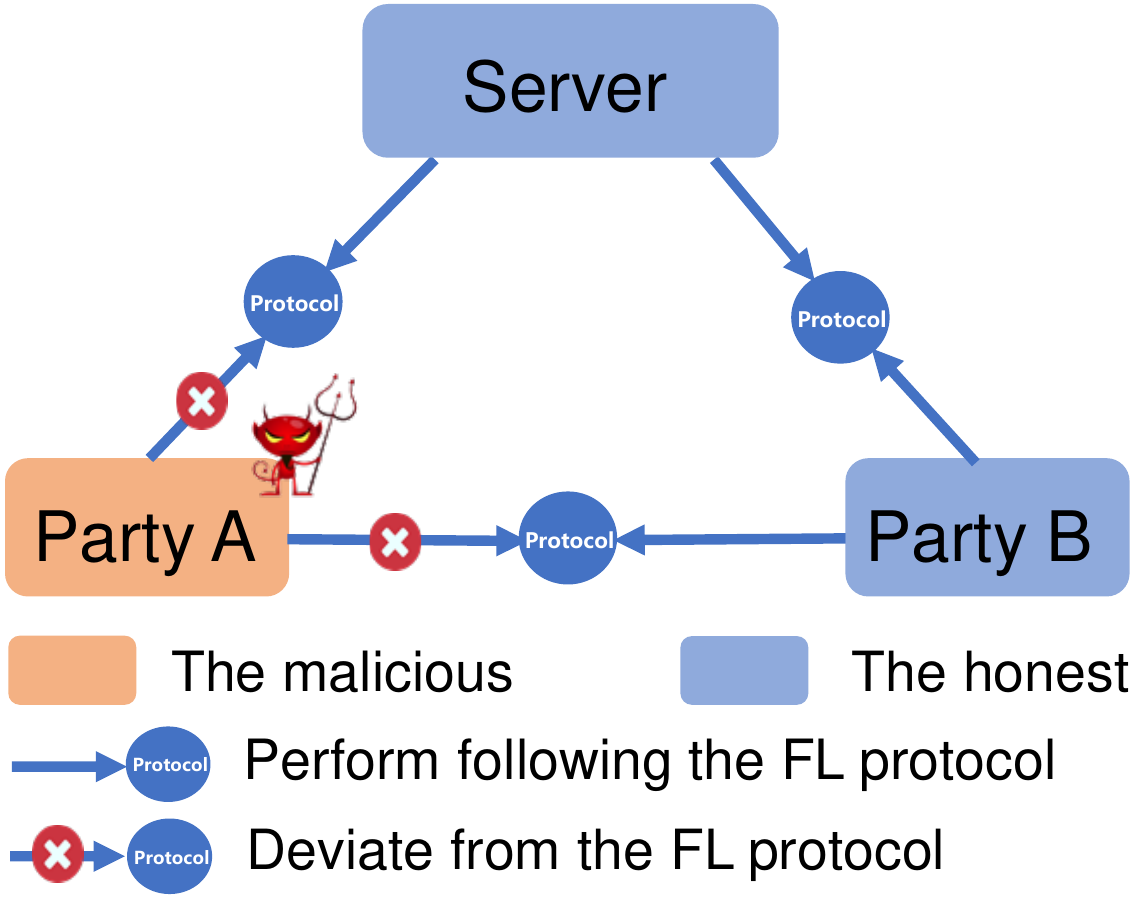}
		\caption{Malicious}
	\end{subfigure}%
	\caption{Illustration of three threat models.}
	\label{fig:threat}
\end{figure*}
\section{A Desirable Companion for VFL}
In this section, we first give a brief review to VFL and then, importantly, we reveal that ZOO is a desirable companion for VFL. Moreover, we give the thorough privacy security analyses of ZOO for VFL (named ZOO-VFL) under existing VFL attacks.
\subsection{Vertical Federated Learning}
Vertical federated learning \cite{gascon2016secure,yang2019federated,hu2019fdml,liu2019communication,gu2020Privacy} is a paradigm for multi-party collaborative learning with privacy preserving. In the VFL system, each party holds different features for one sample.
Specifically, for a VFL system with $q$ parties and training data $\{{\bf{x}}_i,y_i\}_{i=1}^n$, the ${\bf{x}}_i \in \mathbb{R}^{\bar{d}} $ can be represented as a concatenation of all feature blocks, \ie, ${\bf{x}}_i=[x_{i,1}; x_{i,2}; \cdots; x_{i,q}]$, where $x_{i,m} \in \mathbb{R}^{\bar{d}_{m}}$ is stored privately on party $m$, and $\sum_{m=1}^{q}\bar{d}_{ m} = \bar{d}$. Moreover, each party in the VFL system privately maintains and learns a local model, and all parties collaboratively learn the joint model.

Currently, much efforts have been  made towards designing better VFL frameworks for real-world applications from various aspects. In this paper, we summarize these aspects into four metrics, \ie, model applicability, privacy security, communication cost and computation efficiency, which can be used to comprehensively evaluate the performance of VFL frameworks.
Although there have been many works studying VFL following those metrics, to the best of our knowledge, existing VFL frameworks are still not well designed to match those criteria simultaneously. In the following, we reveal that ZOO is a promising choice for designing VFL framework matching these metrics.

\subsection{A Desirable Companion for VFL}

\noindent{\bf Zeroth-Order Optimization:} ZOO methods \cite{huang2020accelerated,huang2019nonconvex} have been developed to effectively solve many ML problems, whose explicit gradient expressions are difficult or infeasible to obtain, such as the structure prediction problems whose explicit gradients are difficult to obtain \cite{sokolov2018sparse}, the bandit and black-box learning problems \cite{shamir2017optimal,liu2018zeroth}, whose explicit gradients are infeasible to obtain.  Specifically, ZOO only uses the function values for optimizing instead of gradients with explicit expressions. Although there have been many works focusing on ZOO, it is still vacant to explore the application of ZOO to VFL, especially, reveal its promising properties for VFL.

In the following we present the promising properties of ZOO-VFL concerning those four practical metrics and reveal that ZOO is a desirable optimization methods for VFL.

\noindent{\bf Model Applicability:} Model applicability is a basic property for the VFL frameworks. Currently, most existing VFL frameworks adopt the gradient-based optimization methods for training. However, frameworks adopting gradient-based optimization methods have the poor model applicability to ML problems whose explicit expressions of gradients are difficult or infeasible to obtain. ZOO only needs the function values for optimizing, which thus is a promising choice to improve the model applicability of VFL to these problem.

\noindent{\bf Privacy Security:} Privacy security is the most important character distinguishing  FL from the distributed learning. Currently, there have many attack models and defense strategies  been proposed \cite{weng2020privacy,luo2020feature,liu2020backdoor}. Especially, the latest proposed inference attacks in \cite{weng2020privacy,luo2020feature,liu2020backdoor} and backdoor attack in \cite{liu2020backdoor} are difficult for existing VFL frameworks to totally defense. Two data inference attacks are proposed in \cite{weng2020privacy}, which, however, require the adversary to access the gradient of the local model and then utilize it for attack. To perform the label inference attack in \cite{liu2020backdoor},
the adversary must be able to access the intermediate  gradient. Similarly, the gradient-replacement backdoor attack proposed in \cite{liu2020backdoor} has to access the intermediate  gradient and then replace it with the targeted one.
In fact, existing attacks that are difficult to defense have to access the informative knowledge such as the model parameters and the gradients.
Thus, to prevent these attacks, one can design a VFL system with the model unknown and without transmitting the informative knowledge between the parties. {\bf{A natural and promising idea to achieve this is letting the model a black box and only transmitting the black-box knowledge}}, such as the function values (the outputs of local and global models).

However, it is impossible to leverage existing optimization methods for VFL to optimize these black-box models when only function values are transmitted. Currently, there have been many optimization methods for black-box learning, such as the Bayesian optimization \cite{karro2017black},  heuristic algorithms \cite{yoo2014modified}, and ZOO \cite{liu2018zeroth}. Among them the ZOO is the optimal choice due to its superiority of theoretical guarantee to heuristic algorithms and less computation complexity than Bayesian optimization. Thus, ZOO is a desired optimization method for improving the privacy security of VFL framework. Especially, since privacy security is considerably important for FL, in the next subsection,  we give the detailed privacy security analyses.
\noindent{\bf Communication Cost and Computation Efficiency:} Note that, in terms of ZOO-VFL, only the function values are necessary to be transmitted. Thus, ZOO-VFL is communication-inexpensive because the PRCO of only transmitting the function values is considerably low. Moreover, we can also design the corresponding asynchronous ZOO algorithm, \emph{i.e.}, AsyREVEL proposed in Section \ref{subsec-asyalgorithm}, that keeps the computation resource being utilized all the time during training for better computation efficiency. Thus, ZOO-VFL is  communication-inexpensive  and computation-efficient.

In above analyses, we reveal that ZOO is naturally a desirable optimization method for VFL. Specifically, ZOO-VFL has favourable model applicability (ability to optimize black-box models), provides satisfied privacy security (ability to defense  existing attacks for VFL), support inexpensive communication  (low PRCO), and efficient computation (adopting asynchronous computation). For a strong support to our claim, we compare a VFL framework that adopts ZOO (proposed in Section \ref{sec-framework}) with existing VFL frameworks following these four metrics and show the results in Table~\ref{tab:comparsion}.
\subsection{Privacy Security of ZOO-VFL}\label{sec-security}
In this section, we  detailedly analyze the privacy security of ZOO-VFL under following three types of threat model, which capsule existing attacks for VFL.
We introduce them as follows, whose illustrations are shown in Fig.~\ref{fig:threat}.

\noindent{\bf Honest-but-Curious:} All parties perform operations following the FL protocol but they may try to learn the private information of the other parties based on the accessed knowledge.

\noindent{\bf Honest-but-Colluding:} All parties perform operations following the FL protocol but they may collude  by sharing the accessed knowledge and use it to learn the private information of the other parties.

\noindent{\bf Malicious:} Some (adversarial) parties may perform operations deviating arbitrarily from the FL protocol, and  to learn the private information of other honest parties or inject a backdoor task by modifying, re-playing, or even removing transmitted messages.

Importantly,  we have the  theorem for the privacy security of VFL.
\begin{theorem}\label{thm-security}
	ZOO for vertical federated learning can defense existing VFL attacks under honest-but-curious, honest-but-colluding, and malicious threat models.
\end{theorem}

\begin{proof}	
	\noindent{\bf Honest-but-curious:} Under this setting, only inference attacks can be performed by leveraging the intermediate computational results. Specifically, the feature inference attack is considered in \cite{yang2019federated,gu2020Privacy}, where the adversary maintains the intermediate computational results of $w^{\mathrm{T}}x_i=z_i$ and uses them to infer $w^{\mathrm{T}}$ and $x_i$. While, this attack will fail in ZOO-VFL because of the inability of solving $n$ equations in more than $n$ unknowns  \cite{du2004privacy,yang2019federated,gu2020Privacy}. The label inference attack is proposed in \cite{liu2020backdoor}, which need access  the intermediate gradient $g_i=\frac{\partial{L}}{\partial{H_i}}$. The adversary uses the element values of $g_i$ and formula $g_i=\frac{\partial{L}}{\partial{H_i}}$ to refer the label of sample $i$. As for ZOO-VFL, no knowledge about the intermediate gradients is exposed, thus it can prevent such attack totally.
	
	\noindent{\bf Honest-but-colluding:} Under this setting, the feature inference attacks (FIA) and the reverse multiplication attack (RMA) are proposed in \cite{luo2020feature} and \cite{weng2020privacy}, respectively.
	In the FIA proposed in \cite{luo2020feature} is performed  the adversary party is supposed to have  its own input $x_{\mathrm{adv}}$, its local model $\theta_{\mathrm{adv}}$, local model of the target party $\theta_{\mathrm{target}}$,  and the final  prediction $z$.  And then it uses the formula $x_{\mathrm{adv}} \cdot \theta_{\mathrm{adv}} + x_{\mathrm{target}}\cdot\theta_{\mathrm{target}}=z$  to infer the feature of the target party $x_{\mathrm{target}}$ during the model prediction stage. Moreover, the generative regression network is also designed in \cite{luo2020feature} for such inference attack, which uses a generative regression network to iteratively approximate the original sample based on multiple model predictions. This attack seems very suitable for the ZOO because it also only uses the model outputs (the predictions). However, the strong primary assumption of both inference attacks that the adversary knows the local model of the target party does not hold in ZOO-VFL, where the local models are private and black-box. Thus, ZOO-VFL can prevent both FIAs totally.  In the RMA, the adversary party accesses the intermediate computational results of successive training epoches, \ie, $w_{t-1}^{\mathrm{T}}x_i$ and $w_{t}^{\mathrm{T}}x_i$,  and the gradient $g_t$, and then uses the iterative gradient-based update rule $w_{t}^{\mathrm{T}}x_i-w_{t-1}^{\mathrm{T}}x_i=-\eta g_tx_i$ to infer $x_i$ ($\eta$ is the learning rate). ZOO-VFL can prevent such RMA totally due to not transmitting the gradients necessary for such attack.
	
	\noindent{\bf Malicious:} Under this setting, the reverse sum attack and backdoor attack are proposed in \cite{weng2020privacy} and \cite{liu2020backdoor}, respectively. In the former, the adversary party encodes a magic number\footnote{https:// en.wikipedia.org/ wiki/Magic number (programming).} into the ciphertext of the first and second gradients (this operation revolves re-playing the gradient), which is used as the global unique identifier to infer the partial orders of training data.  The targeted backdoor task is to assign an attacker-chosen label to input
	data with a specific pattern (i. e. , a trigger) \cite{liu2020backdoor}. Specifically, the adversary party records the received intermediate  gradient of the target sample (denoted as $g_{\mathrm{rec}}$) and replaces the intermediate  gradient of the poisoned sample with $g_{\mathrm{rec}}$. As introduced, both reverse sum and backdoor attacks require the adversary to access the intermediate  gradient. While, ZOO-VFL does not transmit the intermediate gradients  necessary  for these attacks, thus can prevent the reverse sum and backdoor attacks totally.
	
	Thus, we have that ZOO-VFL can  defense existing VFL attacks and protect the privacy security.
	This completes the proof.
\end{proof}
\begin{figure}[!t]
	\centering
	\begin{subfigure}{0.85\linewidth}
		\includegraphics[width=\linewidth]{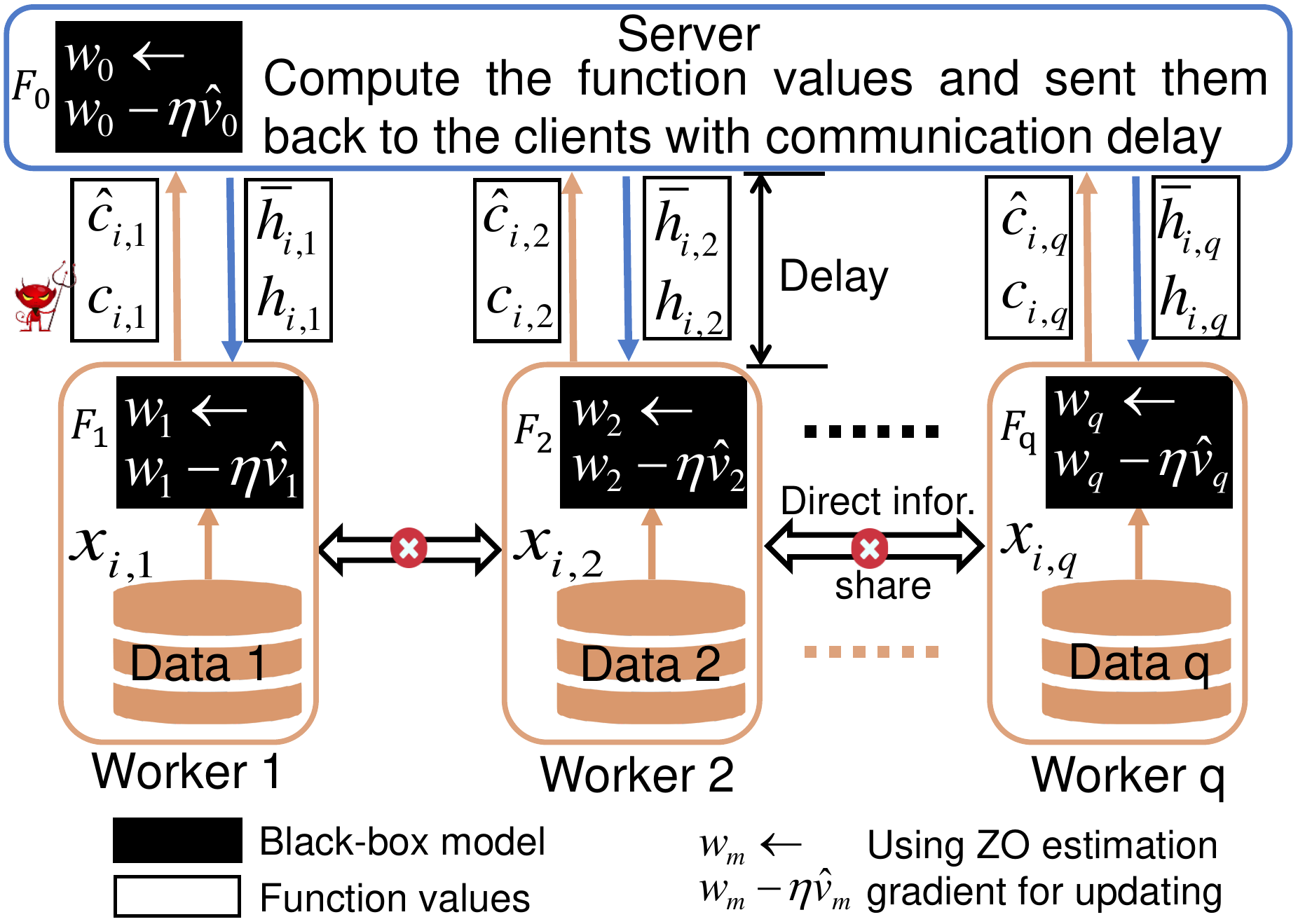}
		\label{struca}
	\end{subfigure}%
	\caption{A diagram of the proposed framework with black-box models, where only black-box knowledge (function values) is transmitted between parties and  exposed to the adversary.}
	\label{struc}
\end{figure}
In fact, all existing VFL  attacks \cite{luo2020feature,liu2020backdoor,weng2020privacy} that are difficult 
to defense have to access the informative knowledge, \ie,  the model parameters or the (intermediate) gradients. While, for ZOO-VFL, only the  black-box knowledge (function values) are exposed.
Thus, it can defense these attacks.
Moreover, it can also prevent the potential VFL attacks that have to access such informative knowledge.

\section{Practical VFL Framework and the AsyREVEL Algorithms}\label{sec-framework}
In this section,  we propose a novel practical VFL framework with black-box models and asynchronous ZOO algorithms, which inherits the promising properties of ZOO, and is inseparably interconnected by above analyses of ZOO-VFL.

\subsection{Generalized Form of VFL}
This paper considers a generalized  VFL system with $q$ parties and a server, where each party owns the vertically partitioned feature data and the server (maybe a party or trusty third-party) owns the labels. In this VFL system, all parties and the server want to solve a finite-sum problem in the following composite form
\begin{subequations}\label{P1}
	\begin{align}\label{P}
		f(w_0,{{\bf w}}) :=\frac{1}{n}\underbrace{\sum_{i=1}^{n} {{F_0\left(w_0,c_{i,1},\cdots,c_{i,q}; y_{i}\right)}} + \lambda \sum_{m=1}^{q}g(w_m)}_{f_{i}(w_0,{{\bf w}})} \quad
 \text{with} \quad c_{i,m}=F_m(w_m;x_{i,m}) \quad \forall m\in[q]\tag{P}
	\end{align}
\end{subequations}
where $f_i(w_0,{{\bf w}}):\mathbb{R}^{\bar{d}}\rightarrow \mathbb{R}$ is the cost function of the $i$-th sample,
${{\bf w}}=\{w_1,\cdots,w_q\} \in \mathbb{R}^{d}$, $w_m \in \mathbb{R}^{d_m}$ for $m\in [q]$ (given a positive integer $q$, $[q]$ denotes a set $\{1,\cdots,q\}$) defines a local  model $F_m$ on party $m$, which maps input $x_{i,m}$ to output $c_{i,m}$, $w_0\in \mathbb{R}^{d_0}$ defines a global model ${F_0}$ learned and maintained by the server, $d = \sum_{m=0}^{q}d_m$, and $g$ is the regularized function.
Especially, problem (\ref{P}) is a generalized form that capsules a wide range of machine learning models. Two examples are shown as follows.

\noindent{{\bf Generalized Linear Model:}} For $m \in [q]$, $ F_m$ can be a linear model, such as $\text{F}_m(w_m;x_{i,m})=w_m^{\mathrm{T}} x_{i,m}$. In this case, if we choose ${\text{F}_0}(c_i;y_i)={\text{log}}(1+e^{-y_i\sum_{m=1}^{q}c_{i,m}}) $ for binary classification tasks, Problem (\ref{P}) will reduce to the classical logistic regression model. We can also choose suitable ${F_0}$ to obtain other  linear models such as  linear regression and support vector machine.

\noindent{{\bf Neural Network Model:}} For $m\in[q]$, $F_m$ can also be a nonlinear model such as neural networks. In this case, $c_{i,m}$ is presented in the following composite form
\begin{subequations}
	\begin{align}
		&\text{input layer:} \ u_0 = x_{i,m}, \quad  \text{and}\ \text{output layer:}\ c_{i,m} = u_K \\
		& \text{intermediate layers:}\ u_l = \sigma_l(h_lu_{l-1} + b_l), \ \ l=1,\cdots,K
	\end{align}
\end{subequations}
where $\sigma_l$ is an active function with linear or nonlinear form, $h_l$ and $b_l$ for $l\in[K]$ correspond to the parameter  $w_m$, $K$ is the number of layer.
In this case, ${F_0}$ can be either a simple network, e.g., the fully connection networks or other complicated deep neural networks.

\subsection{Practical Vertical Federated Learning Framework with Black-Box Models}
Aiming at the generalized VFL problem in the form of (\ref{P}), we propose a novel VFL framework with black-box models, whose diagram  is presented in Fig.~\ref{struc}. As illustrated, the whole data are  vertically  stored on each party locally and privately.  Especially, the local models and the global model are black-box models, which are privately maintained and learned by the parties and server, respectively. Moreover, each local model cascades to the global model and all local models are connected by this global model. Information such as model parameter and data sharing between parties is prohibited, which thus can prevent the data and model from directly leaking. Importantly, the function values transmitted between all parties and the server is black-box knowledge, which is useful to defense existing attacks for VFL (refer to Section \ref{sec-security}). In the following, we present how to propose the AsyREVEL algorithms.

\subsection{AsyREVEL Algorithms}
\begin{algorithm}[!t]
	\caption{AsyREVEL SGD}\label{algo-zosgd}
	\begin{algorithmic}[1]
		\item initialize variables  for workers $m \in [q]$
		\WHILE{ not convergent}
		\STATE {\bf when} {\bf client} $m$ is activated, {\bf do}:
		\STATE \ \ Sample an index $i\overset{\text{Unif}}{\sim} [n]$
		\STATE \ \ Compute $c_{i,m}$, $\hat{c}_{i,m}$ and  upload them to the server
		\STATE \ \  Receive $h_{i,m}$ and  $\bar{h}_{i,m}$ from the server (in a listen manner)
		\STATE \ \  Compute $\whv_m =  \hat{\nabla}_{m}f_i(w_0,\bar{{{\bf w}}})$
		\STATE \ \ Update $w_m\leftarrow w_m-\eta_m \whv_m $
		\STATE {\bf when server} receives $c_{i,m}$ and $\hat{c}_{i,m}$, {\bf do}:
		\STATE \ \ Compute $h_{i,m}$, $\bar{h}_{i,m}$, $\hat{h}_{i,m}$, and sent $h_{i,m}$, $\bar{h}_{i,m}$ to client $m$
		\STATE \ \  Compute $\whv_0 =  \hat{\nabla}_{0}f_i(w_0,\bar{{{\bf w}}})$
		\STATE \ \ Update $w_0\leftarrow w_0-\eta_0 \whv_0 $
		\ENDWHILE
	\end{algorithmic}
\end{algorithm}
\label{subsec-asyalgorithm}
Given a function $F_i(\bar x)$,  a typical two-point stochastic gradient estimator for ZOO is defined as
\begin{align}\label{eq:twopoint}
	\hat{\nabla} F_i(\bar x)=\frac{d_{\bar x}}{{\mu}}[F_i(\bar x + {\mu} u_i) - F_i(\bar x)]u_i
\end{align}
where $\bar x\in \mathbb{R}^{d_{\bar x}}$, ${\mu}>0$ is the smoothing parameter,  and random directions $\{u_i\}$ are i.i.d. drawn from a specific distribution.

However, it is difficult to apply this zeroth-order estimation (ZOE) technique to the VFL due to the much different problem form and application scenario. Specifically, the models to be optimized are distributed over the parties and the server but in a composite form. As shown in Fig.~\ref{struc}, each local model cascades to the global model and all local models are connected by this global model, which is much different from the existing problem forms for ZOO. This  leads to a challenging problem of  designing a proper ZOO algorithm  for our proposed VFL framework.

In this paper, we apply the ZOE technique to each model separately, \ie, to estimate $\frac{\partial F_0}{\partial w_m}$, $m=0,1,\cdots,q$. Because if we take all black-box models (both local and global) as a whole and then apply ZOE technique to estimate $\frac{\partial F_0}{\partial [w_0,{{\bf w}}]}$, we can not leverage the feature-distributed character of VFL and can only design the synchronous algorithms. Moreover, we use the cascade relation between each $F_m$ ($m\in[q]$) and $F_0$ to compute the function value of $F_0$, and then use it to compute the zeroth-order estimation of $\frac{\partial F_0}{\partial w_m}$ directly. Note that we do not apply ZOE technique to  $\frac{\partial F_0}{\partial F_m}$ and $\frac{\partial F_m}{\partial w_m}$ separately, and then use the chain rule, \ie, $\frac{\partial F_0}{\partial w_m} = \frac{\partial F_0}{\partial F_m}\frac{\partial F_m}{\partial w_m}$, to compute the ZOE of $\frac{\partial F_0}{\partial w_m}$. Because the multiplication of two ZOE of gradient  will introduce extra variance.

Motivated by the above analyses and  Eq.~(\ref{eq:twopoint}), we defined the  ZOE of $f_i$ with respect to (w.r.t.) $w_m, \ m=1,\cdots,q$, as
\begin{align}\label{eq:zerogradient}
	\hat{\nabla}_{m} f_i(w_0,{{\bf w}})=\frac{d_{m}}{{\mu_m}}[f_i(w_m + {\mu_m} u_{i,m}) - f_i(w_m)]u_{i,m}
\end{align}
where $f_i(w_m + {\mu_m} u_{i,m}) = f_i(w_0,\cdots, w_m + {\mu_m} u_{i,m},\cdots)$ denotes function $f_i(w_0,{{\bf w}})$ with the other parameters fixed and only $w_m$ as the variable, $d_m$ is the dimension of $w_m$, ${\mu_m}>0$ is the smoothing parameter, and $\{u_{i,m}\}$ are i.i.d. random directions drawn from different distributions. For notation brevity, we define that $c_i = \{c_{i,m}\}_{m=1}^q$ contains function values of sample $i$ from all parties. And $f_i(w_m)$ and $f_i(w_m + {\mu_{m}} u_{i,m})$  are computed as follows.
\begin{align}\label{eq-1}
	f_i(w_m) &={F_0}(w_0,c_i)+ \lambda g(w_m)
	= {h}_{i,m}  + \lambda g(w_m),\nonumber\\
	f_i(w_m + {\mu_m} u_{i,m}) &={F_0}(w_0,c_{i,-m})+ \lambda g(w_m+{\mu_m} u_{i,m})
	\nonumber\\
	& = \bar{h}_{i,m}  + \lambda g(w_m+{\mu_m} u_{i,m})
\end{align}
where $c_{i,-m} = \{\{c_{i,j}\}_{j\in[q],j\ne m},\hat{c}_{i,m}\}$ means $c_i$ with $c_{i,m}$ being replaced by $\hat{c}_{i,m}=F_m(w_m+{\mu_m} u_{i,m};x_{i,m})$.
For $w_0$, there is
\begin{align}\label{eq-2}
	\hat{\nabla}_{0}f_i = \frac{d_{w_0}}{{\mu_m}}(\hat{h}_{i,m} - h_{i,m})u_{i,m},
\end{align}
where $\hat{h}_{i,m} = {F_0}(w_0 + {\mu_m} u_{i,m},c_i)$.

\noindent{\bf AsyREVEL algorithm:}
The proposed AsyREVEL algorithm under our VFL framework is shown in Algorithm \ref{algo-zosgd}.
At step 4, the activated party $m$ computes $c_{i,m}$ and $\hat{c}_{i,m}$ using its private data and local model and then sent them to the server. When the server receives $c_{i,m}$ and $\hat{c}_{i,m}$ from party $m$, it uses them  together with the other parties' function values received previously (stored in the server) to compute $h_{i,m}$, $\bar{h}_{i,m}$ and $\hat{h}_{i,m}$.  Note that those function values of  the other $q-1$ parties are steal due to the asynchronously updating. At step 9, the server then uses $h_{i,m}$ and $\hat{h}_{i,m}$ to compute the ZOE  of $\nabla_{0}f_i$ following Eq.~(\ref{eq-2}).  For client $m$, it needs to query the server for the values of $h_{i,m}$ and  $\bar{h}_{i,m}$  and then uses them to compute the ZOE of local gradient at step 6. Note that, $\bar{{{\bf w}}}$ used at step 6 is the steal state of ${{\bf w}}$  because of both the asynchronous updates and communication delay. An auxiliary illustration of these steps is shown in Fig.~\ref{struc}.

Moreover, we consider two different AsyREVEL algorithms, \ie, AsyREVEL-Gau and -Uni. Specifically, the algorithmic steps of them are the same as those of Algorithm \ref{algo-zosgd}, while the random directions used in Eqs.~(\ref{eq:zerogradient}) and (\ref{eq-2}) are i.i.d. drawn from a zero-mean isotropic multivariate Gaussian distribution for AsyREVEL-Gau and a uniform distribution over a unit sphere for AsyREVEL-Uni.

\section{Convergence Analysis and Complexity Analysis}
In this section, we provide the convergence and complexity analyses of our proposed AsyREVEL algorithms. Note that we only give the sketch of convergence analysis and one can refer to the arXiv version of this paper for the details. First we present some  preliminaries necessary for the convergence analysis.
\begin{assum}\label{assum0}
	Function $f$ is bounded below that is,
	\begin{align}
		f^*:=\inf_{[w_0,{{\bf w}}]\in \mathbb{R}^d} f(w_0,{{\bf w}}) > -\infty.
	\end{align}
\end{assum}
\begin{assum}\label{assum1}
	For $f_i$, $i=1,\ldots,n$ in problem (\ref{P}), we assume the following conditions hold:\\
	\noindent {\bf{Lipschitz Gradient:}} $\nabla f_i$ is $L$-Lipschitz continuous, i.e., there exists a constant $L$ for $ \forall \ [w_0,{{\bf w}}], [w_0', {{\bf w}}']$  such that
	\begin{equation}
		\|\nabla f_i (w_0,{{\bf w}}) - \nabla f_i (w_0',{{\bf w}}')\| \le L\|[w_0,{{\bf w}}]-[w_0',{{\bf w}}']\|. \nonumber
	\end{equation}
	and there exists an $L_m>0$ for $m=0,\cdots,q$ such that $\nabla_{m} f_i$ is $L_m$-Lipschitz continuous.
	
	\noindent {\bf{Bounded Block-Coordinate Gradient:}} For $m=0,\cdots,q$, there exists a constant $\sigma_m$ such that  $\|\nabla_{m} f_i(w_0,{{\bf w}})\|^2\leq \sigma_m^2$.
\end{assum}
Above assumptions are standard in previous optimization works \cite{zhang2021faster,huang2019faster,huang2019nonconvex,huang2019faster1},
where Assumption \ref{assum0} guarantees the feasibility of problem (\ref{P}), Assumption \ref{assum1} imposes (block-coordinate) smoothness on the individual functions and introduces bounded block-coordinate  gradients.
We also introduced Assumption~\ref{assum2} to handle the asynchronous updates, which is helpful for tracking the behavior of the global model.
\begin{assum}\label{assum2}
	The activated client $m_t$ is independent of $m_0$, $\cdots,$ $m_{t-1}$ and satisfies $\mathbb{P}(m_t=m):=p_m$
\end{assum}
Moreover, the function values of  the other $q-1$ parties used to compute $h_{i,m}$ (or $\hat{h}_{i,m}$, $\bar{h}_{i,m}$) are steal due to the asynchronously updating manner and possible communication delay. To handle this case, we introduce the following assumption to bound the delay.
\begin{assum}\label{assum3}{\bf{Bounded Delay:}} For $\bar{w}_t $ that is the $w$ used for computing at current iteration $t$, there is
	\begin{equation}\label{eq-assum3}
		\bar{{{\bf w}}}^t = {{{\bf w}}}^{t-\tau_{t}^{n,m}} = {{{\bf w}}}^t + \eta_{m_{t'}} \sum_{t'\in D^\prime(t)}\widehat{v}^{t'}_{m_{t'}},
	\end{equation}
	where $D'(t)=\{t-1,\cdots,t-\tau_{t}^{n,m}\}$ is a subset of previous iterations and $\tau_{t}^{n,m}\leq \tau$.
\end{assum}

\subsection{Convergence Analyses}
\begin{figure*}[!t]
	\centering
	\begin{subfigure}{0.24\linewidth}
		\includegraphics[width=\linewidth]{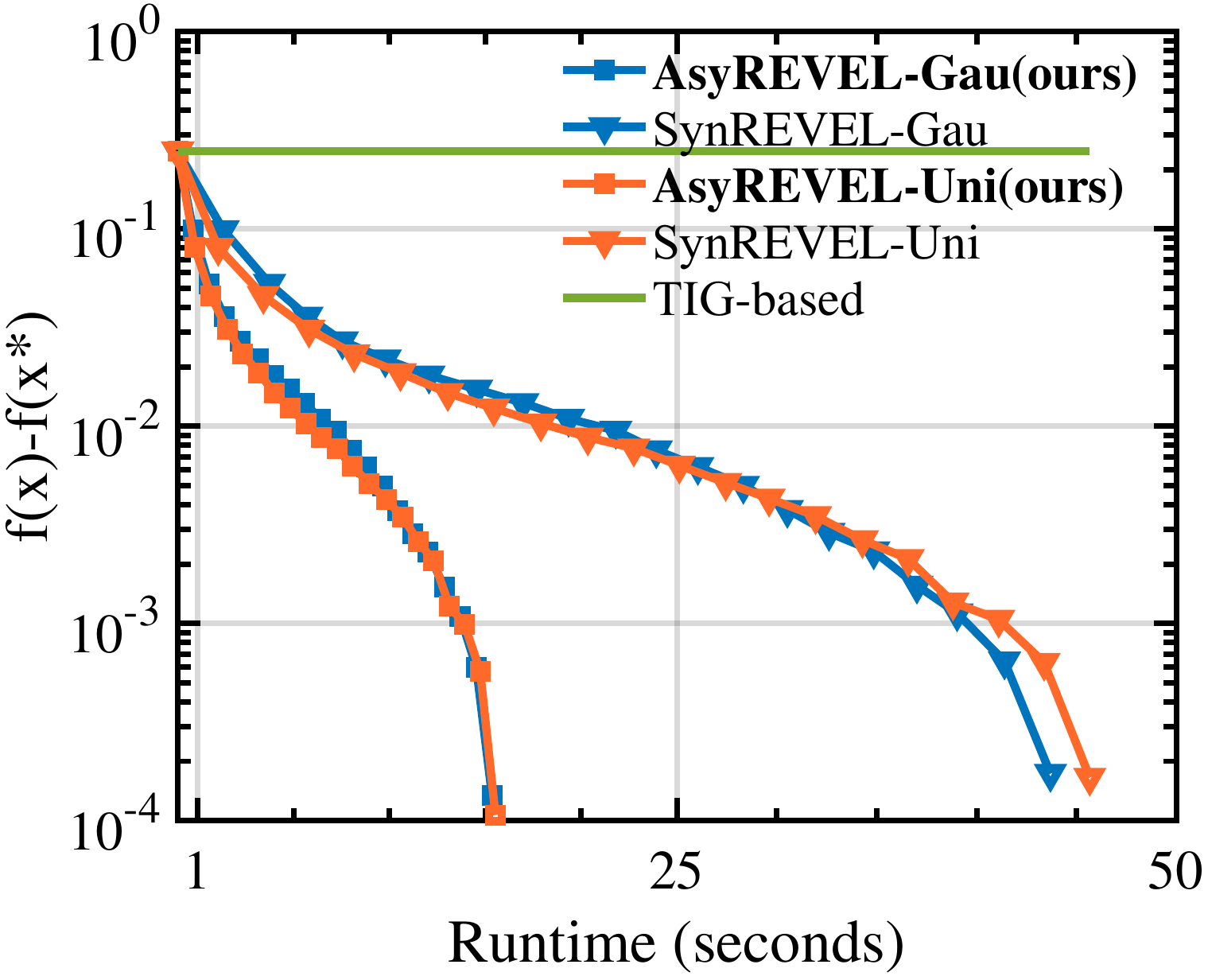}
		\caption{Data: $D_1$}
	\end{subfigure}
	\
	\begin{subfigure}{0.24\linewidth}
		\includegraphics[width=\linewidth]{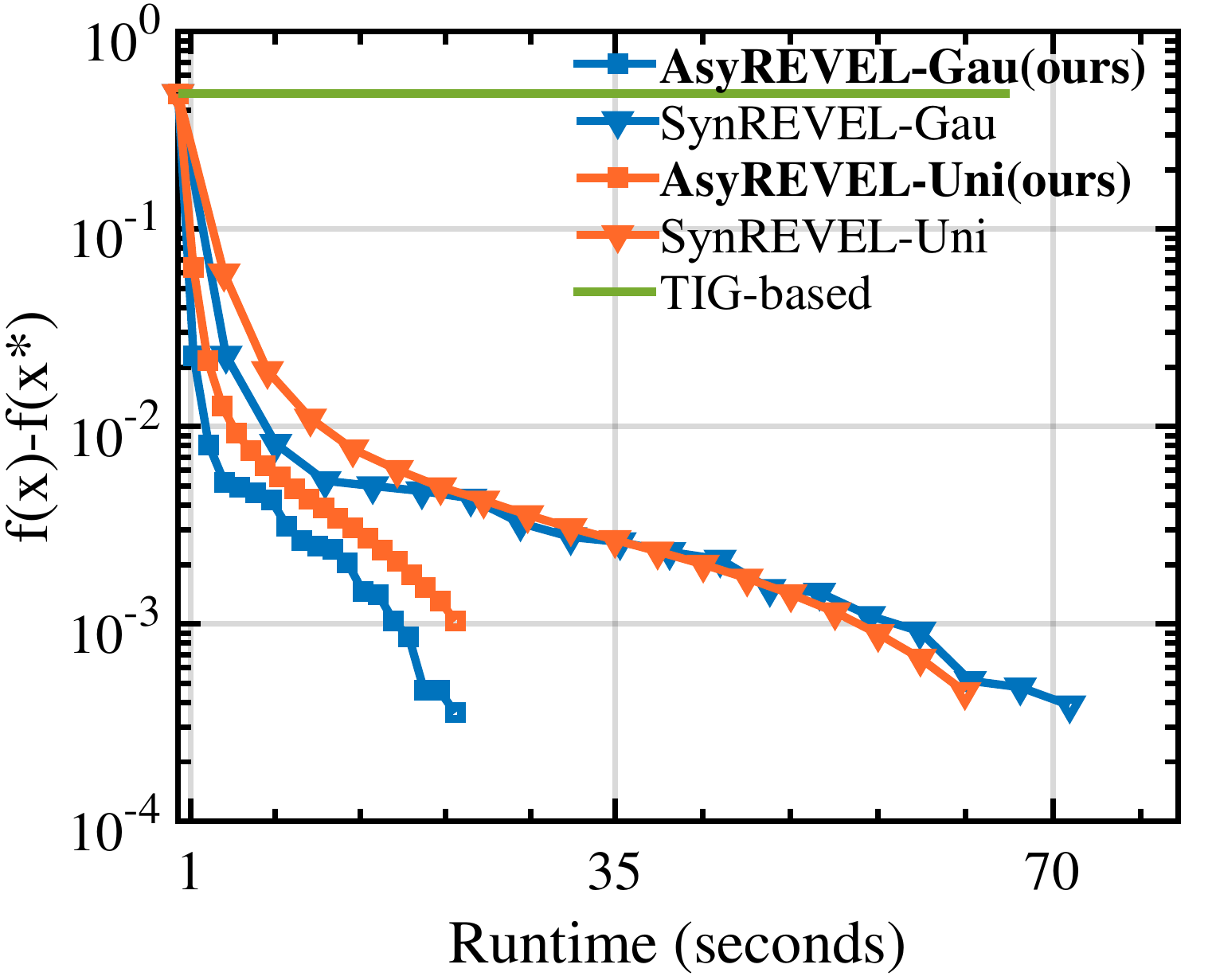}
		\caption{Data: $D_2$}
	\end{subfigure}
	\
	\begin{subfigure}{0.24\linewidth}
		\includegraphics[width=\linewidth]{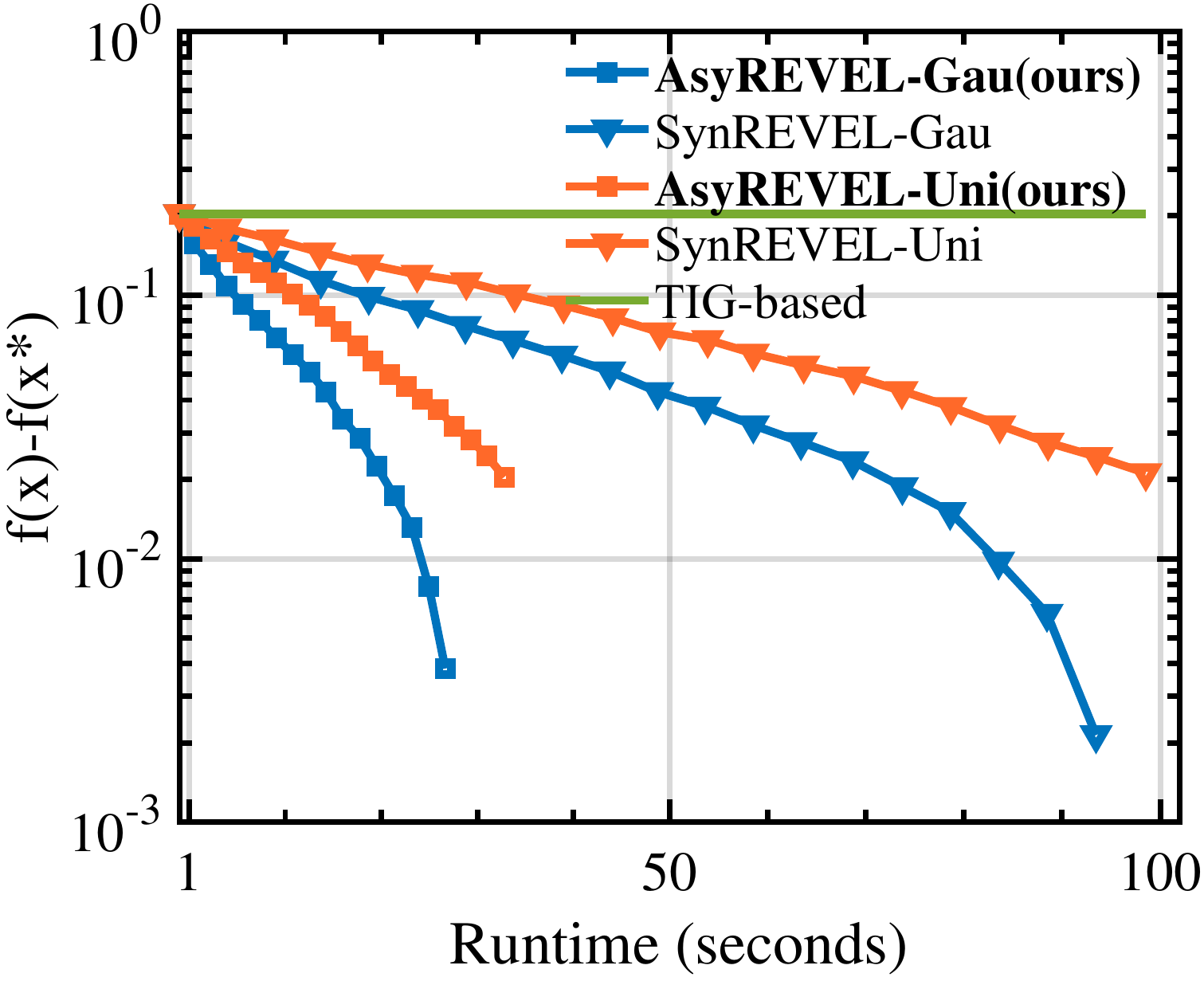}
		\caption{Data: $D_3$}
	\end{subfigure}%
	\
	\begin{subfigure}{0.24\linewidth}
		\includegraphics[width=\linewidth]{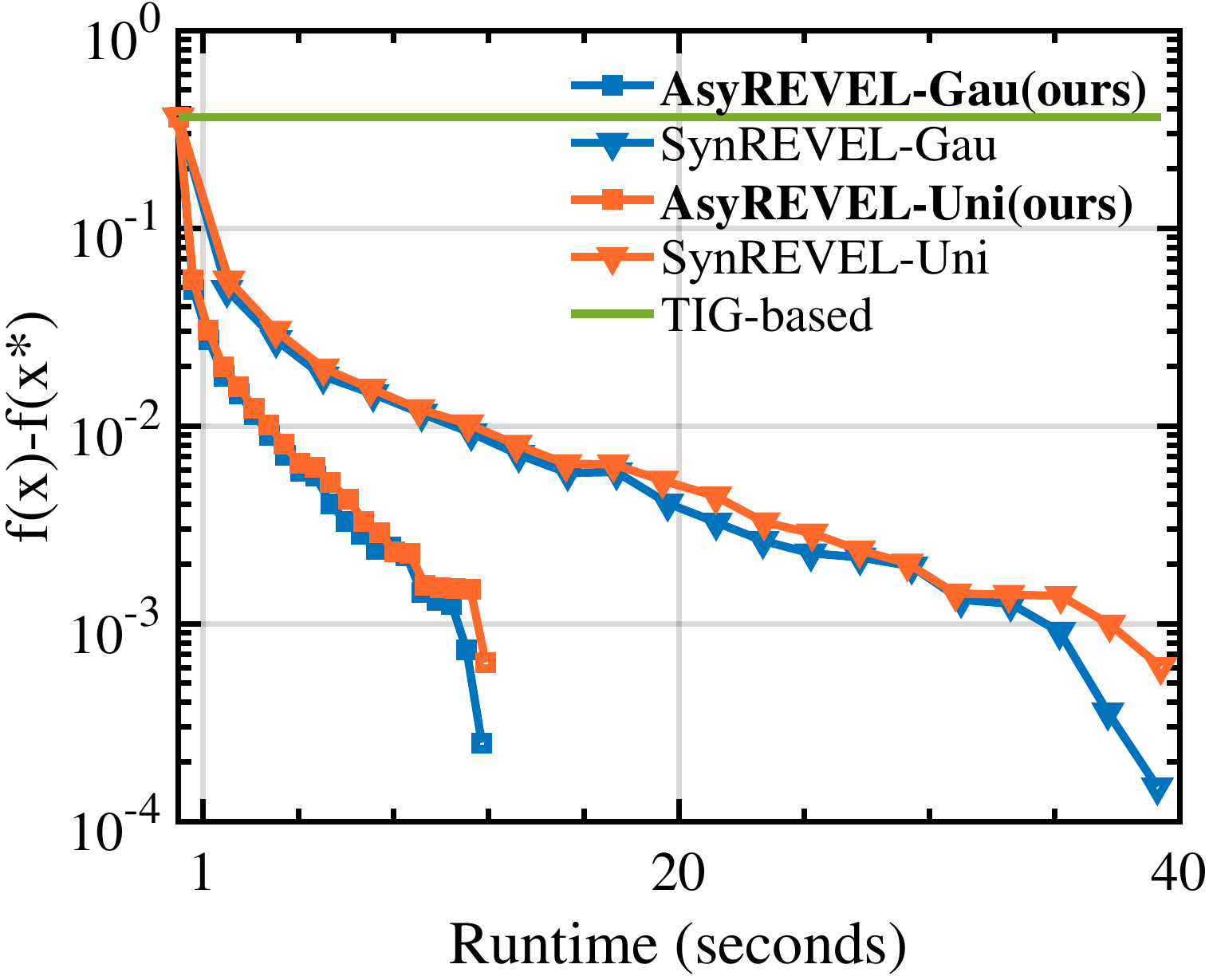}
		\caption{Data: $D_4$}
	\end{subfigure}
	\\
	\begin{subfigure}{0.24\linewidth}
		\includegraphics[width=\linewidth]{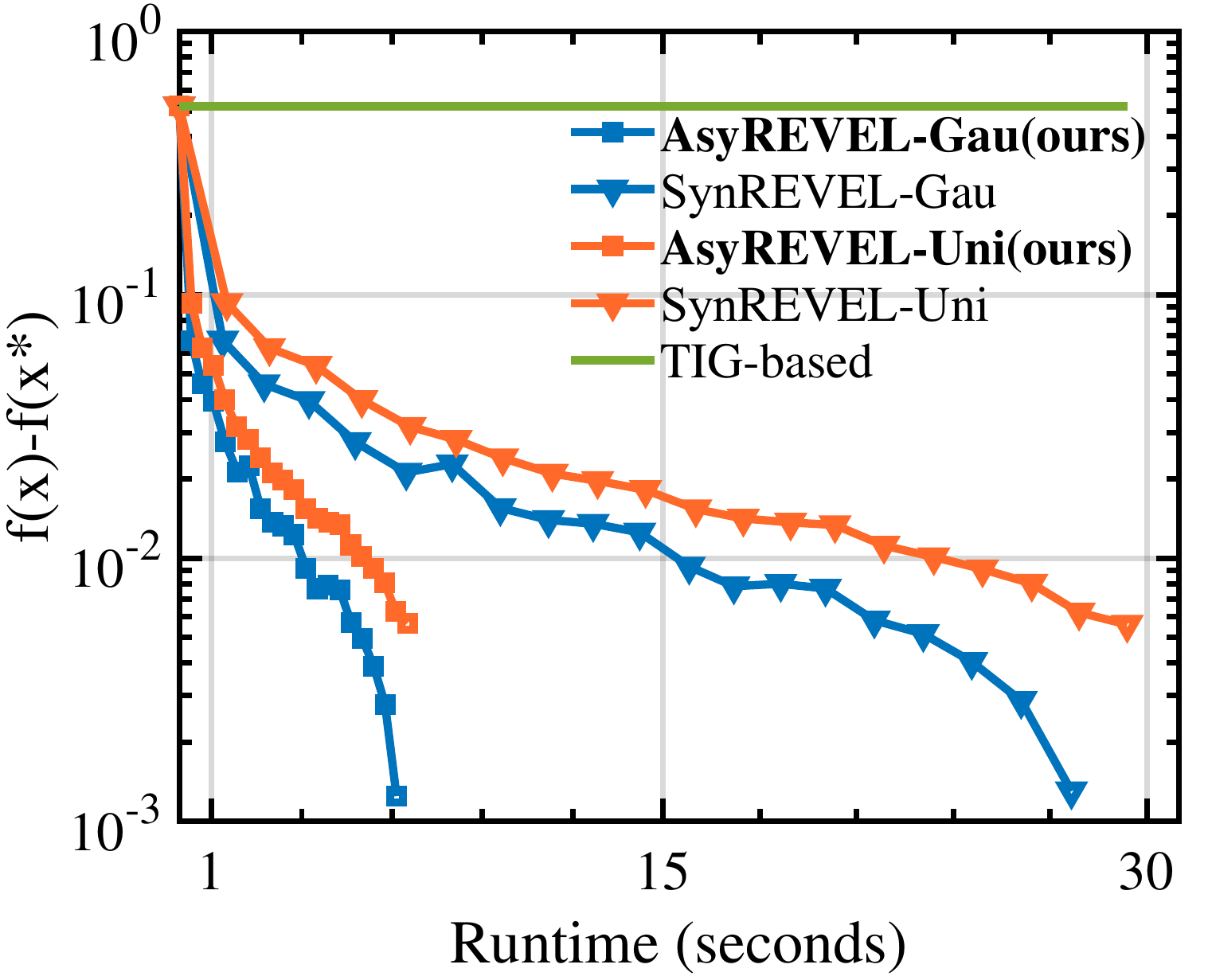}
		\caption{Data: $D_5$}
	\end{subfigure}
	\
	\begin{subfigure}{0.24\linewidth}
		\includegraphics[width=\linewidth]{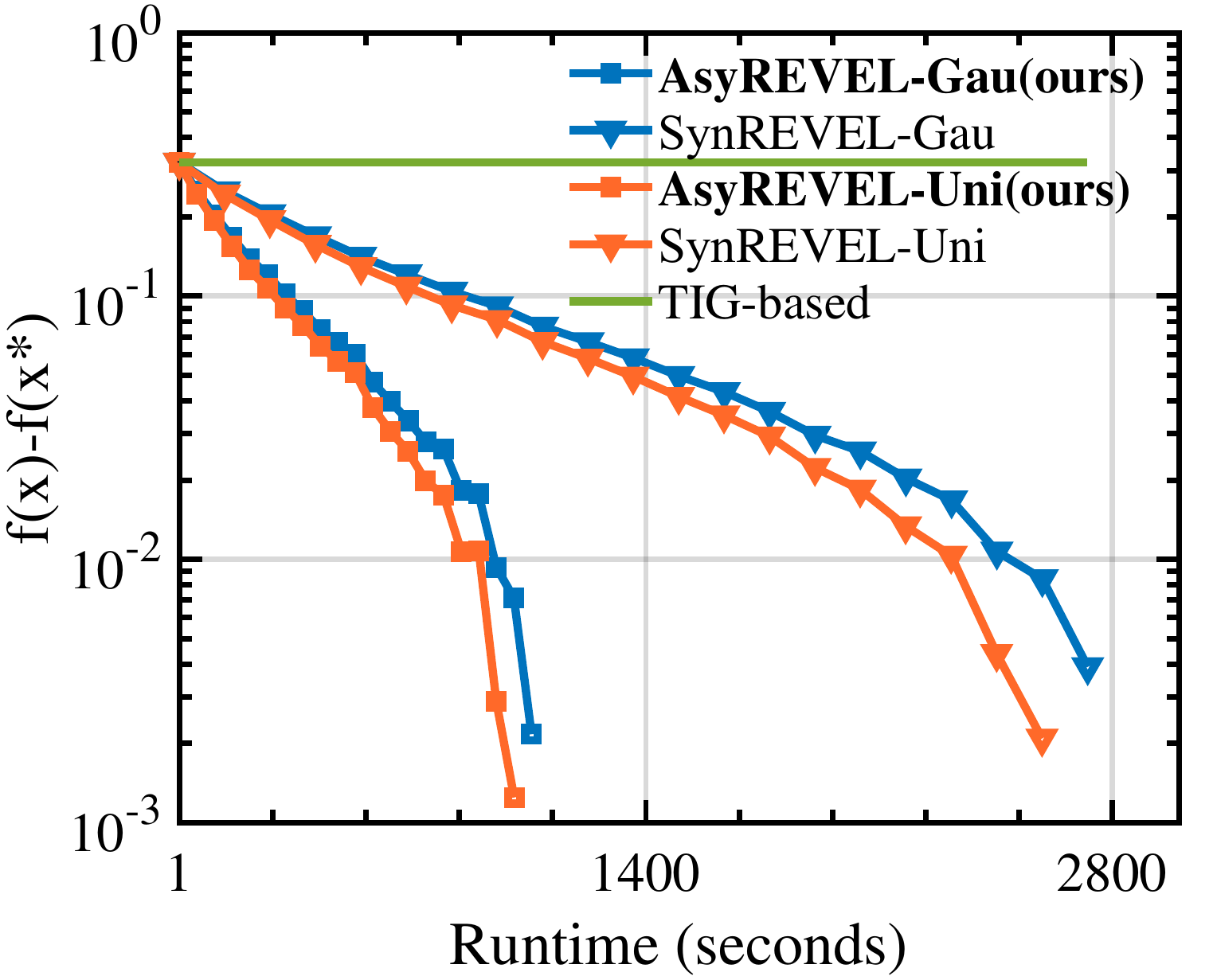}
		\caption{Data: $D_6$}
	\end{subfigure}
	\
	\begin{subfigure}{0.24\linewidth}
		\includegraphics[width=\linewidth]{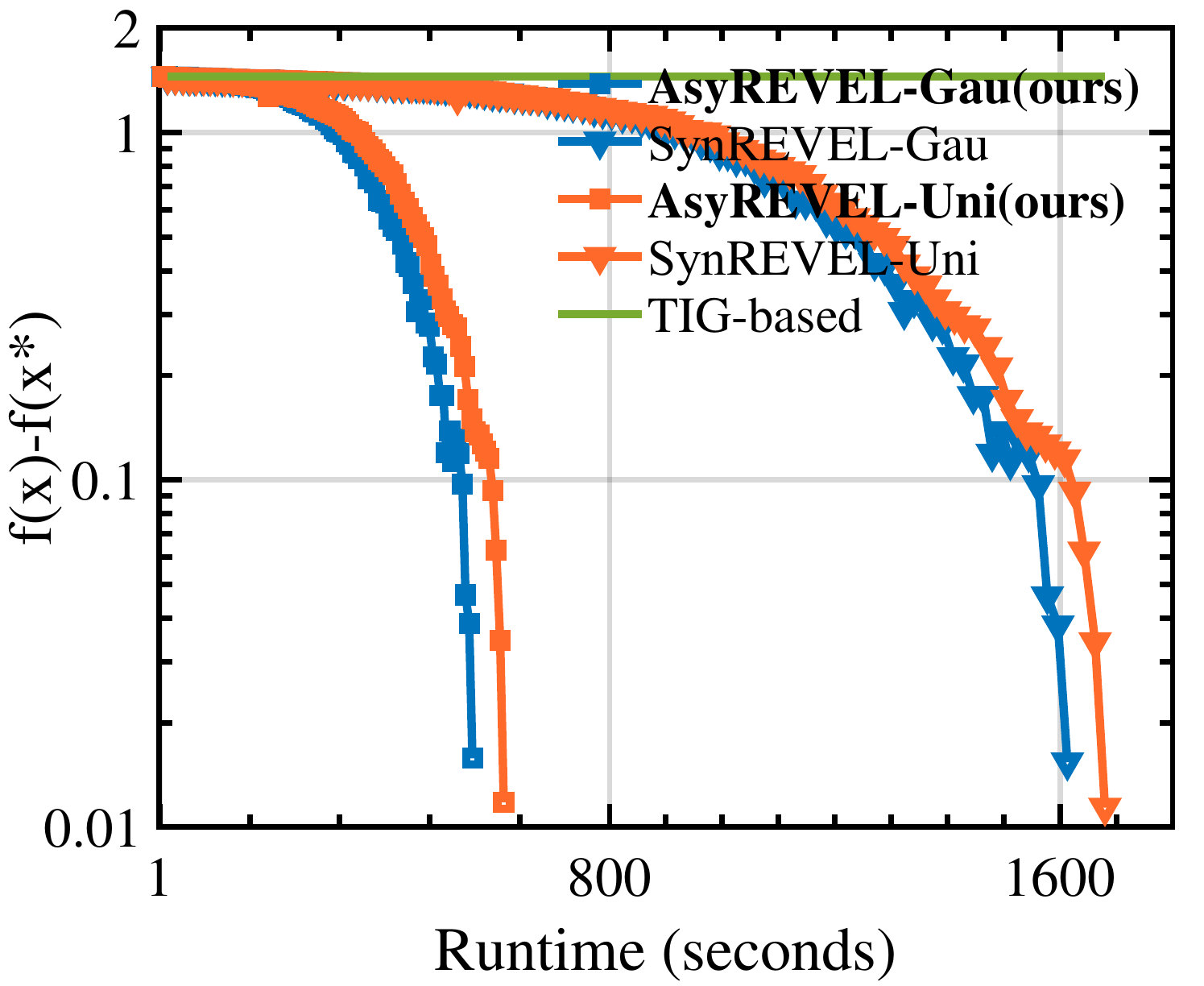}
		\caption{Data: $D_7$}
	\end{subfigure}%
	\
	\begin{subfigure}{0.24\linewidth}
		\includegraphics[width=\linewidth]{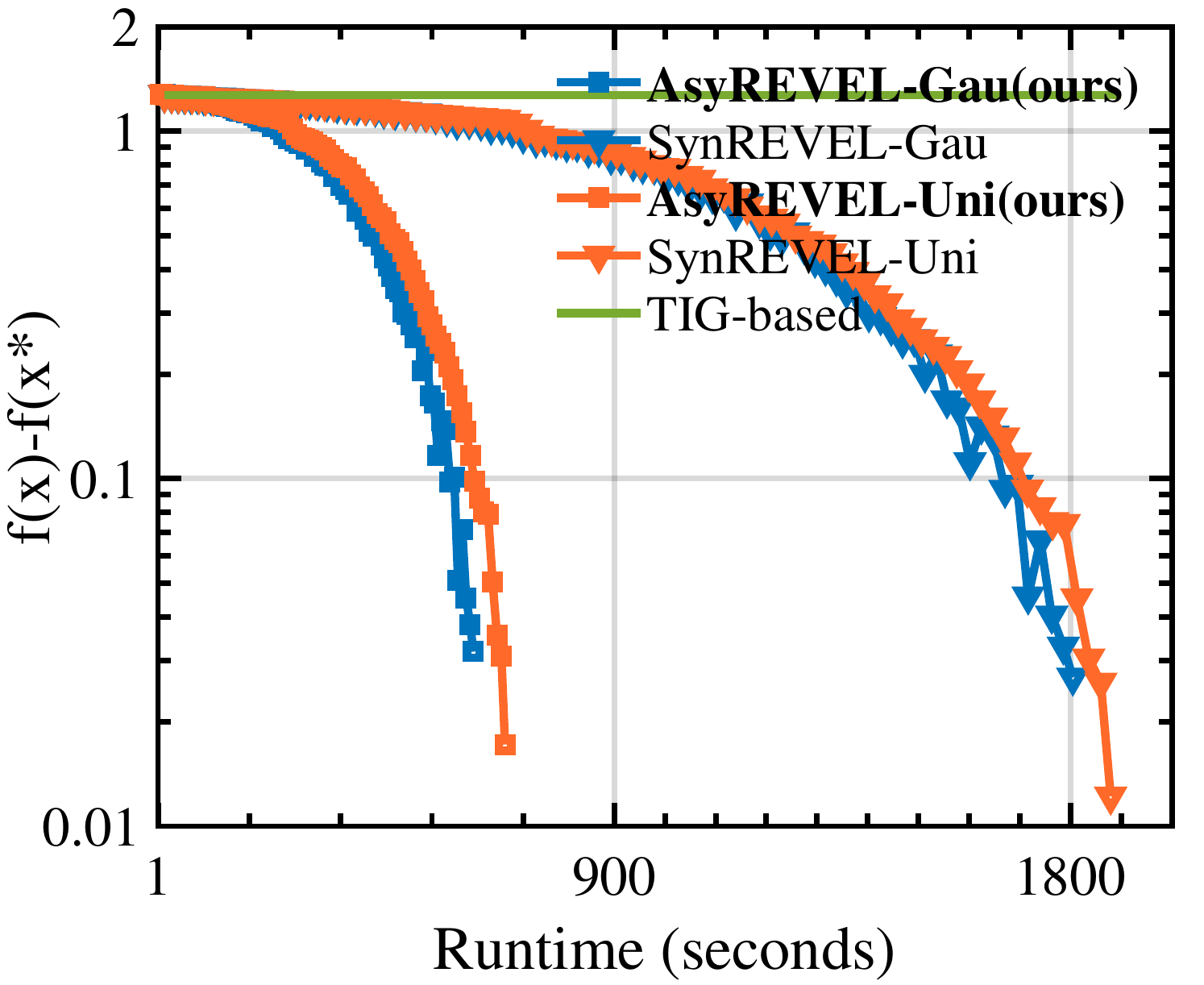}
		\caption{Data: $D_8$}
	\end{subfigure}
	\caption{Results for solving black-box federated learning problem on different datasets.}
	\label{Exp-ncon}
\end{figure*}
\begin{theorem}\label{thm-gau}
	Under Assumptions~\ref{assum0}-\ref{assum3},
	to solve problem~\ref{P} with AsyREVEL-Gau,
	let $\eta=\min\{\frac{1}{4(\tau+1)L},\frac{m_0}{\sqrt{T}}\}$ with constant $m_0>0$ and ${\mu_m}=\mathcal{O}(\frac{1}{\sqrt{T}})$ such as ${\mu_m} = \frac{1}{\sqrt{T}L_*d_*^{3/2}}$, then we have
	\begin{align}\label{eq-coro-gau}
		\frac{1}{T}\sum_{t=0}^{T-1}\BE\|\nabla f(w_0,{{\bf w}})\|^2
		&\leq \frac{4p_*(f^0-f^*)}{\sqrt{T}m_0}
		+ \frac{{8{p_*}m_0(L+\tau L)\sigma_*^2}}{\sqrt{T}}
		\nonumber \\
		&
		+\frac{{(q+1)}}{2T^2L_*d_*^2}
		+  \frac{(q+1)+3p_*}{2T}
	\end{align}
	where $d_* = \max_{m}d_m+3$, $p_*=\min_mp_m$, $\tau$ is independent of $T$.
\end{theorem}

\begin{theorem}\label{thm-uni}
	Under Assumptions~\ref{assum0}-\ref{assum3},
	to solve problem~\ref{P} with AsyREVEL-Uni,
	let $\eta=\min\{\frac{1}{4(\tau+1)L},\frac{m_0}{\sqrt{T}}\}$ with constant $m_0>0$ and ${\mu_m}=\mathcal{O}(\frac{1}{\sqrt{T}})$ such as ${\mu_m} = \frac{1}{\sqrt{T}L_*d_*}$, then we have
	\begin{align}\label{eq-coro-gau}
		\frac{1}{T}\sum_{t=0}^{T-1}\BE\|\nabla f(w_0,{{\bf w}})\|^2
		&\leq \frac{4p_*(f^0-f^*)}{\sqrt{T}m_0}
		+ \frac{{8{p_*}m_0(L+\tau L)\sigma_*^2}}{\sqrt{T}}
		\nonumber \\
		&
		+\frac{{(q+1)}}{2T^2L_*d_*^2}
		+  \frac{(q+1)+3p_*}{2T}
	\end{align}
	where $d_* = \max_{m}d_m$, $p_*=\min_mp_m$, $\tau$ is independent of $T$.
\end{theorem}

\begin{remark}\label{remark-gau}
	Under Assumptions~\ref{assum0}-\ref{assum3}, given the parameters in corresponding theorems, the convergence rates of both AsyREVEL-Gau and -Uni are $\mathcal{O}(\frac{1}{\sqrt{T}})$.
\end{remark}

\subsection{Complexity Analyses}
The total computation complexity at steps 4, 6 and 7 is $\bo(d_m)$, and that at steps 10, 11 and 12 is  $\bo(d_0)$. Thus, the whole computation complexity of Algorithm~\ref{algo-zosgd} is $\bo(d_m+d_0)$. Importantly, only the local outputs and global outputs are transmitted between the party $m$ and the server, and the total communication complexity of Algorithm~\ref{algo-zosgd} is  $ \bo(1)$. Thus, our framework is communication-inexpensive compared with those transmitting the (intermediate) gradients.
\begin{table*}[!t]
	\centering
	\caption{Dataset Descriptions.}
	\label{dataset}
	\begin{tabular}{@{}ccccccccc@{}}
		\toprule
		\multirow{2}{*}{}
		& \multicolumn{6}{c}{For logistic regression task} & \multicolumn{2}{c}{For deep learning task}  \\ \cmidrule(l){2-7} \cmidrule(l){8-9}
		& $D_1$ & $D_2$ & $D_3$ & $D_4$  & $D_5$ & $D_6$ & $D_7$ & $D_8$\\
		\midrule
		\#Samples & 24,000 & 96,257 & 677,399 & 32,561
		&45,749& 400,000 & 60,000 & 60,000 \\
		\#Features & 90 & 92 & 47,236 & 127
		& 300 & 2,000 & 784 &784  \\ \bottomrule
	\end{tabular}
\end{table*}	
\section{Experiments}
In this section, we implement extensive experiments to demonstrate the model applicability, privacy security, inexpensive communication  and efficient computation  of our proposed algorithms. Moreover, we also show that  AsyREVEL is scalable and lossless.

\noindent {\bf Experiment Settings:} All experiments are performed on a machine with four sockets, and each sockets has 12 cores. The MPI is used for communication. Following previous works, we vertically partition the data into $q$  non-overlapped parts with nearly equal number of features.
An optimal $\eta$ for all client is chosen  from $\{5e^{-1},1e^{-1},\cdots\}$, and the learning rate for server is $\eta/q$.

\noindent {\bf Datasets:}
We use eight datasets for evaluation, which are summarized in Table~\ref{dataset}, among which $D_1$ (UCICreditCard), $D_2$ (GiveMeSomeCredit), $D_3$ (Rcv1), $D_4$ (a9a), $D_5$ (w8a) and $D_6$ (Epsilon) are used for logistic regression problem, $D_7$ (MNIST) and $D_8$ (Fashion MNIST) are used for the deep learning tasks.

\noindent{\bf Framework for Comparison:} We introduce a framework that has the same structure of our framework but directly transmits the intermediate gradient (called TIG-based framework, refer to \cite{liu2020backdoor,vepakomma2018split} for details) instead of the function values. Specifically, in TIG-based framework, intermediate gradient $\frac{\partial F_0}{\partial F_m}$ is computed by the server and transmitted to party $m$, and then party $m$ uses the chain rule, \ie, $\frac{\partial F_0}{\partial w_m} = \frac{\partial F_0}{\partial F_m}\frac{\partial F_m}{\partial w_m}$ to compute the local gradient.

\subsection{Evaluation of Favorable Model Applicability}
To evaluate the model applicability of our proposed VFL framework we introduce two generalized  black-box learning problems for VFL, where only the local and global outputs  are transmitted.

\noindent {\bf Black-Box Federated Logistic Regression:} Specifically, we optimize the nonconvex logistic regression problem
\begin{equation}\label{P-nonconvex}
	\min_{{{\bf w}} \in \mathbb{R}^d} f({{\bf w}}):=\frac{1}{n} \sum_{i=1}^{n}  {\text{log}}(1+e^{-y_i {{\bf w}}^{\mathrm{T}} x_i}) + \frac{\lambda}{2} \sum_{i=1}^{d} \frac{{{\bf w}}_i^2}{1+{{\bf w}}_i^2},
\end{equation}
where $\lambda=1e^{-4}$ for all experiments. This is an example of  generalized linear mode, where $F_m(w_m;x_{i,m})=w_m^\top x_{i,m}$ and ${F_0}(c_i;y_i)={\text{log}}(1+e^{-y_i\sum_{m=1}^{q}c_{i,m}})$. In this case, we set $\eta=0.001$ and $\mu=0.001$ for all experiments.

\noindent {\bf Black-Box Federated Neural Network:} For this case, we train a fully connected network (FCN)-based model. Specifically, the local embedding mode is a 2-layer FCN ($784\times 128$ and $128\times 1$) with nonlinear activation function (ReLU) and the global model is a 1-layer ($q\times 10$) FCN and a softmax layer. In this case, we set $\eta=0.002$ for AsyREVEL-Gau and $\eta=7.5$ for -Uni, and set  $\mu=0.001$ for all experiments.

The loss v.s. training time results in Figs.~\ref{Exp-ncon} show that our framework can solve the black-box models while the TIG-based can not because it can not compute the gradient necessary for updating.

\subsection{Evaluation of Inexpensive Communication}
To demonstrate that our framework is communication-inexpensive, we compute the ratios of time spending on TIG relative to that of transmitting the function values. The corresponding results are listed in Table~\ref{table:communication}, which show that the PRCO of our framework is much lower than that of the TIG-based one, especially when the gradient has high dimension. Note that, AsyREVEL has the convergence rate of $1/\sqrt{T}$ for nonconvex problems, which is the same as that of general asynchronous SGD for nonconvex VFL problems \cite{hu2019fdml}. To further reduce the communication cost, one can adopt the variance reduction techniques \cite{liu2018zeroth} (for better convergence rate) or local SGD technique \cite{liu2019communication} to reduce the number of communication rounds.
\begin{table*}[!t]
	\centering
	\caption{Ratios of time spending, which are obtained during the training process of fixed number (\emph{e.g.}, $n$) of iterations (10 trials).}
	\label{table:communication}
	\setlength{\tabcolsep}{0.75mm}{\small{
			\begin{tabular}{@{}cccccccccc@{}}
				\toprule
				& & $D_1$($d_\ell=12$) & $D_2$($d_\ell=12$) & $D_3$($d_\ell=5904$) & $D_4$($d_\ell=16$)
				& $D_5$($d_\ell=37$) & $D_6$($d_\ell=250$) & $D_7$($d_\ell=98$) & $D_8$($d_\ell=98$)
				\\ \midrule
				&Ratios &1.065 & 1.078 & 5.794 & 1.192
				& 1.192 & 1.824& 1.672 & 1.672 \\
				\bottomrule
	\end{tabular}}}
\end{table*}
\begin{table*}[!t]
	\centering
	\caption{Accuracy of different algorithms to evaluate the losslessness  of our algorithms (10 trials).}
	\label{exp-lossless}
	\setlength{\tabcolsep}{1.0mm}{\small{\begin{tabular}{@{}cccccccccc@{}}
			\toprule
			&Algorithm& $D_1$(\%) & $D_2$(\%) & $D_3$(\%) & $D_4$(\%)
			& $D_5$(\%) & $D_6$(\%) & $D_7$(\%) & $D_8$(\%)
			\\ \midrule
			& NonF & 81.93$\pm$0.36 & 93.50$\pm$0.28 & 95.24$\pm$0.06 & 85.16$\pm$0.08
			& 89.85$\pm$0.08 & 87.79$\pm$0.09 & 91.89$\pm$0.25 & 81.32$\pm$0.11
			\\
			&{\bf AsyREVEL-Gau} & 81.93$\pm$0.24 & 93.50$\pm$0.31 & 95.24$\pm$0.14 & 85.16$\pm$0.08
			& 89.85$\pm$0.10 & 87.79$\pm$0.11 & 91.89$\pm$0.29 & 81.32$\pm$0.15
			\\
			\midrule
			& NonF & 81.88$\pm$0.10 & 93.48$\pm$0.09 & 95.14$\pm$0.12 & 85.14$\pm$0.12
			& 89.88$\pm$0.07 & 87.89$\pm$0.12 & 91.84$\pm$0.32 & 81.45$\pm$0.11
			\\
			&{\bf AsyREVEL-Uni}  & 81.88$\pm$0.14 & 93.48$\pm$0.11 & 95.14$\pm$0.09 & 85.14$\pm$0.09
			& 89.88$\pm$0.12 & 87.89$\pm$0.07 & 91.84$\pm$0.38 & 81.45$\pm$0.09
			\\
			\bottomrule
	\end{tabular}}}
\end{table*}
\subsection{Evaluation of Computation Efficiency}
To demonstrate the efficiency of asynchronous computation, we compare AsyREVEL algorithm with its synchronous counterpart, {\ie} SynREVEL. When implementing the synchronous algorithms, there is a synthetic straggler party which maybe 20\% to 60\% slower than the faster one to simulate the industry application scenario.

{\noindent {\bf Asynchronous Efficiency:}} In these experiments, we set $q=8$.
As  shown in Fig.~\ref{Exp-ncon}, the loss v.s. runtime curves demonstrate that our algorithms are more computation-efficient than the synchronous ones.

{\noindent {\bf Asynchronous Scalability:}} We also consider the asynchronous speedup scalability in terms of $q$. Given $q$ parties, there is
\begin{equation}
	\text{$q$-parties speedup} =\frac{\text{training time  of using  1 party}}{\text{training time of using $q$ parties}},
\end{equation}
where training time is the time spending on reaching a certain precision of sub-optimality, \ie, $5e^{-4}$ for $D_4$. The results  are shown in Fig.~\ref{Exp-sca}, which   demonstrate that our asynchronous algorithms has much better $q$-parties speedup scalability than the synchronous ones and can achieve near linear speedup.

\subsection{Evaluation of Losslessness}
To demonstrate the losslessness of our algorithms, we compare AsyREVEL with its non-federated (NonF) counterpart whose only difference to AstREVEL is that all data are integrated together for modeling. For datasets without testing data, we split the data set into $10$ parts, and use one of them for testing. Each comparison is repeated 10 times with $q=8$, and a same stop criterion, \emph{e.g.}, $5e^{-4}$ for $D_4$. As shown  in Table~\ref{exp-lossless}, the accuracies of our algorithms are the same with those of NonF algorithms.
~
\begin{figure}[!t]
	\centering
	\begin{subfigure}{0.47\linewidth}
		\includegraphics[width=\linewidth]{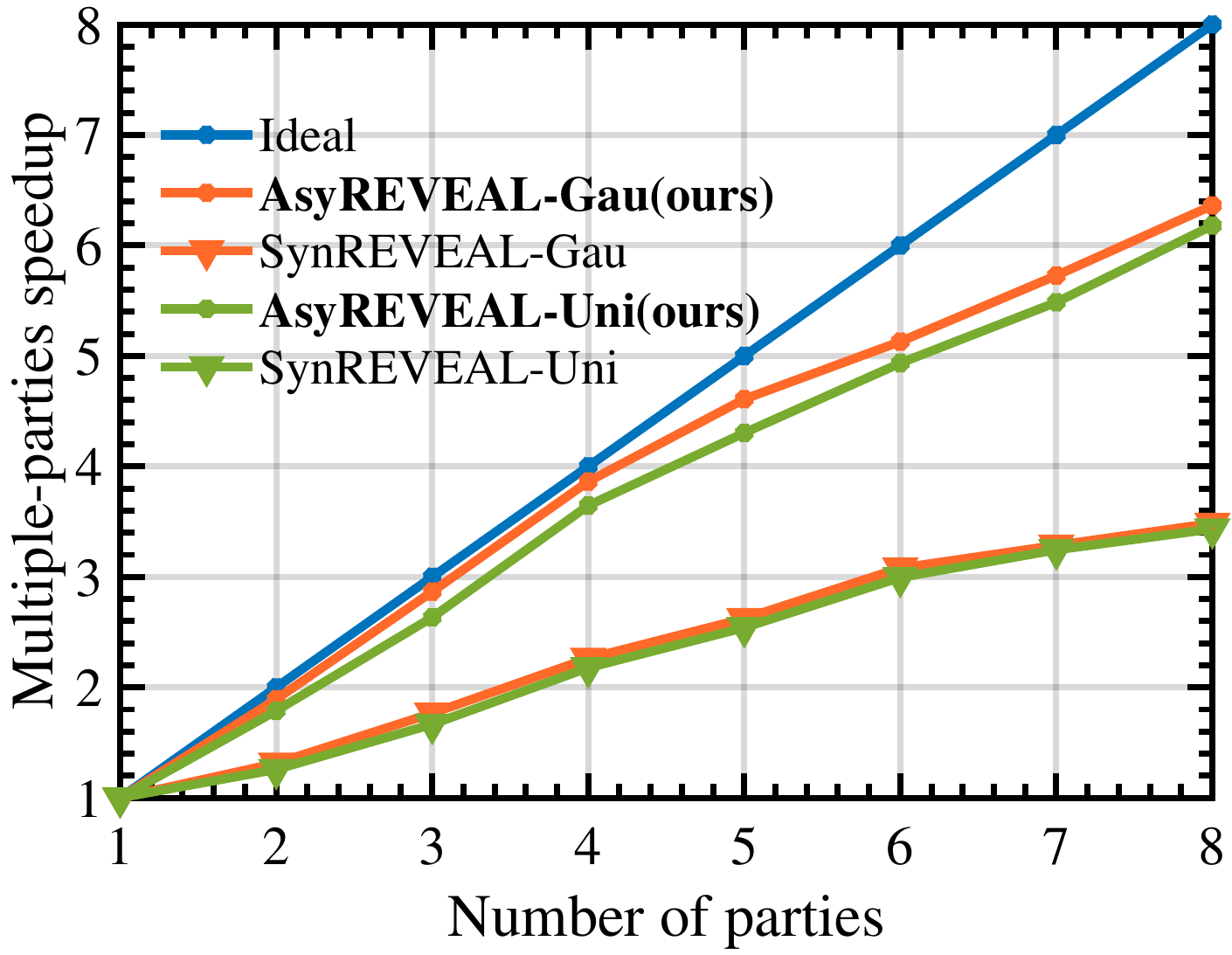}
	\end{subfigure}
	\quad
	\begin{subfigure}{0.47\linewidth}
		\includegraphics[width=\linewidth]{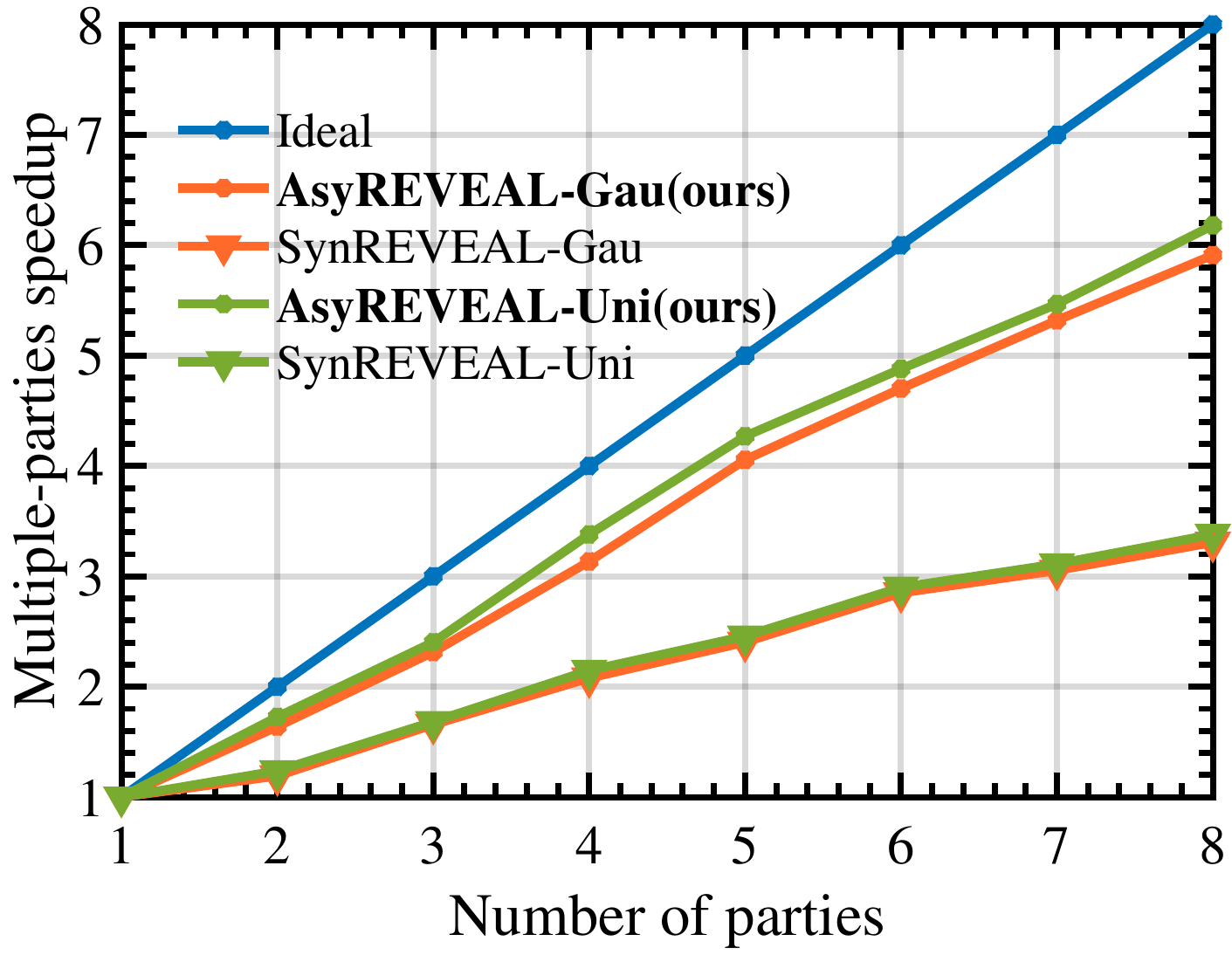}
	\end{subfigure}
	\caption{Linear speedup results for AsyREVEL and SynREVEL algorithms. Left: $D_5$ for federated LR problem. Right: $D_7$ for federated neural network problem.}
	\label{Exp-sca}
\end{figure}
\section{Conclusion}
In this paper, we revealed that ZOO is a desirable companion for VFL. Specifically, ZOO can 1) improve the model applicability of VFL framework. 2) prevent VFL framework from attacks under three levels of threat models, \ie, the curious, colluding, and malicious. 3) support inexpensive communication and efficient computation. We proposed a novel practical VFL framework with black-box models, which inherits the promising properties of  ZOO. Under this framework, we raised the novel AsyREVEL algorithms with two smoothing techniques. Moreover, we prove the privacy security of ZOO-VFL  under different attacks and theoretically drive the convergence rates of AsyREVEL algorithms under nonconvex condition.
\newpage
\appendix
\section*{Appendix}

\begin{lemma}
	\label{lemma: smooth_f_random}
	Suppose that Assumption~\ref{assum1} holds, then  we have
	
	1) $f_{\mu_m}$ is $L_m$-smooth and $f_{\mu}$ is $L$-smooth
	\begin{align}\label{eq: smooth_est_grad}
		\nabla_m f_{\mu_m}  =  \mathbb E_{u_m}  [  \hat \nabla_{m} f(w_0,\BW)  ]
		,
		\nabla f_{\mu}
		=  \mathbb E_{u }
		[ \hat \nabla f(w_0,\BW)]
	\end{align}
	where
	$\hat \nabla_{m} f(w_0,\BW)$ is given by Eq.~(6).
	
	2) For any $w_m \in \mathbb R^{d_{m}}$,
	\begin{align}
		& | f_{\mu_m}( w_m) - f(w_m) | \leq \frac{L_md_m {\mu_m}^2}{2} \label{eq: dist_f_smooth_true_random}  \\
		& \| \nabla_m f_{\mu_m} (w_0,\BW) - \nabla_{m} f(w_0,\BW) \|_2^2 \leq \frac{{\mu_m}^2 L_m^2 (d_{m}+3)^3}{4}, \label{eq: dist_smooth_true} \\
		&
		\mathbb E_{ u} \left [
		\|\hat \nabla_{m} f(w_0,\BW)\|_2^2
		\right ] \leq 2 (d_{m}+4) \| \nabla_{m} f(w_0,\BW) \|_2^2 + \frac{{\mu_m}^2  L_m^2 (d_{m}+6)^3}{2}. \label{eq: up_grad_fsmooth}
	\end{align}
	
	3) For any $w_m \in \mathbb R^{d_{m}}$,
	\begin{align}
		\mathbb E_{\mathbf u} \left [ \| \hat \nabla_{m} f(w_0,\BW) - \nabla_m f_{\mu_m} (w_0,\BW) \|_2^2 \right ]
		\leq 2 (2d_{m}+9) \| \nabla_{m} f(w_0,\BW) \|_2^2 + {{\mu_m}^2  L_m^2 (d_{m}+6)^3}.
		\label{eq: second_moment_grad_random}
	\end{align}
\end{lemma}
\begin{lemma}
	Under Assumptions~\ref{assum0} to \ref{assum3}, for $m=0,1,\cdots,q$ there is
	\begin{align}\label{eq:proof-0}
		\BE\|\whv_\mt^t-v_\mt^t\|^2 &=  \BE\|\whv_\mt^t - \bv_\mt^t + \bv_\mt^t -v_\mt^t\|^2
		\nonumber \\
		&\leq  2\BE\|\whv_\mt^t - \bv_\mt^t\| + 2\BE\|\bv_\mt^t -v_\mt^t\|^2
		\nonumber \\
		&\leq  \frac{{\mu_m}^2L_\mt^2(d_\mt+3)^3}{2} + 2L^2\|\BWW^t -\BW^t\|^2
	\end{align}
	and
	\begin{align}\label{eq:proof-1}
		\BE\|\whv_\mt^t\|^2 &=  \BE\|\whv_\mt^t - v_\mt^t + v_\mt^t\|^2
		\nonumber \\
		&\leq  2\BE\|\whv_\mt^t - v_\mt^t\| + 2\BE\|v_\mt^t\|^2
		\nonumber \\
		&\leq  {{\mu_m}^2L_\mt^2d_\mt^2} + 4L^2\|\BWW^t -\BW^t\|^2 + 2\BE\|v_\mt^t\|^2
		\nonumber \\
		&\leq  3{{\mu_m}^2L_\mt^2(d_\mt+3)^3} + 4L^2\|\BWW^t -\BW^t\|^2 + 4\sigma_\mt^2
	\end{align}
\end{lemma}
\begin{proof}
	
	Under Assumption~\ref{assum1},  and taking expectation w.r.t. the sample index $i_t$, we have
	\begin{align}\label{eq:proof-2}
		&\BE f_{\mu_m}(w_0^{t+1},\BW^{t+1})
		\nonumber \\
		&= \BE \left( f_{\mu_m}(w_0^{t}-\eta_0\whv_0^t
		,\cdots, w_{m_t}^t-\eta_{m_t}\whv_{m_t}^t,\cdots) \right)
		\nonumber\\
		&\leq  \BE \left(  f_{\mu_m}(w_0^{t},\BW^{t}) -\eta_0 \left<V_0^t,\whv_0^t\right> - \eta_{m_t}\left<V_\mt^t,\whv_\mt^t\right>
		+ \frac{L\eta_0^2}{2}\|\whv_0^t\|^2
		+  \frac{L\eta_{m_t}^2}{2}\|\whv_{m_t}^t\|^2 \right)
		\nonumber\\
		&= \BE \left(  f_{\mu_m}(w_0^{t},\BW^{t}) -\eta_0 \left<V_0^t,\whv_0^t - v_0^t + v_0^t\right>
		- \eta_{m_t}\left<V_m^t,\whv_\mt^t -  v_\mt^t + v_\mt^t\right>
		+ \frac{L\eta_0^2}{2}\|\whv_0^t\|^2
		+  \frac{L\eta_{m_t}^2}{2}\|\whv_{m_t}^t\|^2 \right)
		\nonumber\\
		&\leq \BE \left(  f_{\mu_m}(w_0^{t},\BW^{t}) -\eta_0\|V_0^t\|^2 - \eta_0 \left<V_0^t,\whv_0^t - v_0^t \right>
		- \eta_{m_t} \|V_\mt^t\|^2 - \eta_{m_t} \left<V_\mt^t,\whv_\mt^t -  v_\mt^t\right>
		+ \frac{L\eta_0^2}{2}\|\whv_0^t\|^2
		+  \frac{L\eta_{m_t}^2}{2}\|\whv_{m_t}^t\|^2 \right)
		\nonumber\\
		&\leq \BE \left(  f_{\mu_m}(w_0^{t},\BW^{t}) -\frac{\eta_0}{2}\|V_0^t\|^2 + \frac{\eta_0}{2}\|\whv_0^t - v_0^t\|^2
		- \frac{\eta_{m_t}}{2} \|V_\mt^t\|^2 + \frac{\eta_{m_t}}{2}\|\whv_\mt^t -  v_\mt^t\|
		+ \frac{L\eta_0^2}{2}\|\whv_0^t\|^2
		+  \frac{L\eta_{m_t}^2}{2}\|\whv_{m_t}^t\|^2 \right)
		\nonumber\\
		&\leq \BE \left(  f_{\mu_m}(w_0^{t},\BW^{t}) -\frac{\eta_0}{2}\|V_0^t\|^2
		- \frac{\eta_{m_t}}{2} \|V_\mt^t\|^2
		+ (\eta_0 + \eta_\mt + 2L\eta_0^2 + 2L\eta_\mt^2) L^2\|\BWW^t -  \BW^t\|^2\right)
		\nonumber\\
		& + (\frac{\eta_\mt}{4} + \frac{3L\eta_\mt^2}{2}){{\mu_m}^2L_\mt^2(d_\mt+3)^3}
		+ (\frac{\eta_0}{4} + \frac{3L\eta_0^2}{2}){{\mu_m}^2L_0^2(d_0+3)^3} + 2L\eta_\mt^2\sigma_\mt^2 + 2L\eta_0^2\sigma_0^2
	\end{align}
	
	Taking expectation w.r.t. $m_t$, and using Assumption~\ref{assum2}, there is
	\begin{align}\label{eq:proof-3}
		&\BE f_{\mu_m}(w_0^{t+1},\BW^{t+1})\nonumber \\
		&\leq \BE  f_{\mu_m}(w_0^{t},\BW^{t}) -\frac{\eta_0}{2}\BE\|V_0^t\|^2
		-\sum_{m=1}^{q} p_m \frac{\eta_{m}}{2}\BE \|V_m^t\|^2
		+ \underbrace{(\eta_0 + 2L\eta_0^2 + \max_{m}(2L\eta_m^2+\eta_m)) L^2}_{\beta^t} \BE\|\BWW^t -  \BW^t\|^2
		\nonumber\\
		& +{\sum_{m=1}^{q} p_m(\frac{\eta_m}{4} + \frac{3L\eta_m^2}{2}){\mu_m}^2L_m^2}  (d_m+3)^3
		+ (\frac{\eta_0}{4} +\frac{3 L\eta_0^2}{2}){\mu_m}^2L_0^2(d_0+3)^3
		+ {\sum_{m=1}^{q} p_m2L\eta_m^2}\sigma_m^2 + 2{L\eta_0^2} \sigma_0^2
	\end{align}
	According to Assumption~\ref{assum3}, there is
	\begin{align}\label{eq:proof-4}
		\|\BWW^t -  \BW^t\|^2 = \|\sum_{i \in D(t)}\BW^{i+1}-\BW^{i}\|^2 \leq \tau\sum_{i=1}^{\tau}\|\BW^{t+1-i}-\BW^{t-i}\|^2
	\end{align}
	We than bound the term $ \|\BWW^t -  \BW^t\|^2$. First, for $\BE\|\BW^{t+1} -  \BW^t\|^2$
	\begin{align}\label{eq:proof-5}
		\BE \|\BW^{t+1} -  \BW^t\|^2  = \BE\eta_\mt^2\|\whv_\mt^2 \|
		\leq \sum_{m=1}^{q}p_m \eta_m^2(3{{\mu_m}^2L_m^2(d_m+3)^3} + 4\sigma_m^2) + \max_{m}\eta_m^2 4L^2\|\BWW^t -\BW^t\|^2
	\end{align}
	Define a Lyapunov function as
	\begin{align}\label{eq:proof-6}
		M^t = f_{\mu_m}(w_0^t,\BW^t) + \sum_{i=1}^{\tau}\theta_i\|\BW^{i+1}-\BW^{i}\|^2
	\end{align}
	Following Lemma~\ref{lemma: smooth_f_random} and Eq.~\ref{eq:proof-6}, there is
	\begin{align}\label{eq:proof-7}
		&\BE (M^{t+1} - M^t) \nonumber \\
		&= \BE\left( f_{\mu_m}(w_0^{t+1},\BW^{t+1}) + \sum_{i=1}^{\tau}\theta_i\|\BW^{t+1+1-i}-\BW^{t+1-i}\|^2
		- f_{\mu_m}(w_0^t,\BW^t) - \sum_{i=1}^{\tau}\theta_i\|\BW^{t+1-i}-\BW^{t-i}\|^2 \right)
		\nonumber \\
		& =  -\frac{\eta_0}{2}\BE\|V_0^t\|^2
		-\sum_{m=1}^{q} p_m \frac{\eta_{m}}{2}\BE \|V_m^t\|^2
		+ \beta^t\BE\|\BWW^t -  \BW^t\|^2
		+ {\sum_{m=1}^{q} p_m2L\eta_m^2}\sigma_m^2 + 2{L\eta_0^2} \sigma_0^2
		\nonumber \\
		& +{\sum_{m=1}^{q} p_m(\frac{\eta_m}{4} + \frac{3L\eta_m^2}{2}){\mu_m}^2L_m^2}  (d_m+3)^3
		+(\frac{\eta_0}{4} + \frac{3L\eta_0^2}{2}){\mu_m}^2L_0^2 (d_0+3)^3
		\nonumber \\
		& + \theta_1\BE \|\BW^{t+1}-\BW^t\|^2 + \sum_{i=1}^{\tau-1}(\theta_{i+1}-\theta_i)\BE \|\BW^{t+1-i}-\BW^{t-i}\|^2  - \theta_\tau\BE \|\BW^{t+1-\tau}-\BW^{t-\tau}\|^2
		\nonumber \\
		& \leq
		-\frac{\eta_0}{2}\BE\|V_0^t\|^2
		-\sum_{m=1}^{q} p_m \frac{\eta_{m}}{2}\BE \|V_m^t\|^2
		+ \beta^t\tau\sum_{i=1}^{\tau}\|\BW^{t+1-i}-\BW^{t-i}\|^2
		+ {\sum_{m=1}^{q} p_m2L\eta_m^2}\sigma_m^2 + 2{L\eta_0^2} \sigma_0^2
		\nonumber \\
		& +{\sum_{m=1}^{q} p_m(\frac{\eta_m}{4} + \frac{3L\eta_m^2}{2}){\mu_m}^2L_m^2}  (d_m+3)^3
		+ (\frac{\eta_0}{4} + \frac{3L\eta_0^2}{2}){\mu_m}^2L_0^2 (d_0+3)^3
		\nonumber \\
		&
		+ \sum_{i=1}^{\tau-1}(\theta_{i+1}-\theta_i)\BE \|\BW^{t+1-i}-\BW^{t-i}\|^2  - \theta_\tau\BE \|\BW^{t+1-\tau}-\BW^{t-\tau}\|^2
		\nonumber \\
		& + \theta_1(\sum_{m=1}^{q}p_m \eta_m^2(3{{\mu_m}^2L_m^2d_m^2} + 4\sigma_m^2) + \max_{m}\eta_m^2 4L^2\tau\sum_{i=1}^{\tau}\|\BW^{t+1-i}-\BW^{t-i}\|^2)
		\nonumber \\
		& \leq
		-\frac{\eta_0}{2}\BE\|V_0^t\|^2
		-\sum_{m=1}^{q} p_m \frac{\eta_{m}}{2}\BE \|V_m^t\|^2
		+ {\sum_{m=1}^{q} p_m\eta_m^2}(L+4\theta_1)\sigma_m^2
		+ 2{L\eta_0^2} \sigma_0^2
		\nonumber \\
		& + \sum_{m=1}^{q}p_m(\frac{\eta_m}{4}+\frac{3L\eta_m^2}{2}+ 3\theta_1\eta_m^2 ){\mu_m}^2L_m^2  (d_m+3)^3
		+(\frac{\eta_0}{4} + \frac{3L\eta_0^2}{2}){\mu_m}^2L_0^2 (d_0+3)^3
		\nonumber \\
		&
		+ \sum_{i=1}^{\tau-1}(\beta_t\tau + \tau \theta_1\max_{m}\eta_m^2 4L^2 + \theta_{i+1}-\theta_i)\BE \|\BW^{t+1-i}-\BW^{t-i}\|^2  + (\beta_t\tau +  \tau \theta_1\max_{m}\eta_m^2 4L^2 -\theta_\tau)\BE \|\BW^{t+1-\tau}-\BW^{t-\tau}\|^2
	\end{align}
	If we choose $\eta_0,\eta_m\leq \bar{\eta}\leq\frac{1}{4(L+2\theta_1)}$, then there is $\beta^t\leq\frac{3\bar{\eta}L^2}{2}$. Then for Eq.~\ref{eq:proof-7} there is
	\begin{align}\label{eq:proof-8}
		&\BE (M^{t+1} - M^t) \nonumber \\
		& \leq -\frac{1}{2}\min\{\eta_0,p_m\eta_m\}\BE\|\nabla f_{\mu_m}(w_0,\BW)\|^2
		+ {\sum_{m=1}^{q} p_m\eta_m^2}(L+4\theta_1)\sigma_m^2
		+ 2{L\eta_0^2} \sigma_0^2
		\nonumber \\
		&  + \sum_{m=1}^{q}p_m(\frac{\eta_m}{4}+\frac{3L\eta_m^2}{2}+ 3\theta_1\eta_m^2 ){\mu_m}^2L_m^2  (d_m+3)^3
		+(\frac{\eta_0}{4} + \frac{3L\eta_0^2}{2}){\mu_m}^2L_0^2 (d_0+3)^3
		\nonumber \\
		&
		-\sum_{i=1}^{\tau-1}(\theta_i - \theta_{i+1} - {\frac{3}{2}\bar{\eta}L^2}\tau - 4\tau\theta_1L^2\bar{\eta}^2) \BE \|\BW^{t+1-i}-\BW^{t-i}\|^2
		- (\theta_\tau - {\frac{3}{2}\bar{\eta}L^2}\tau - 4\tau\theta_1L^2\bar{\eta}^2)\BE \|\BW^{t+1-\tau}-\BW^{t-\tau}\|^2
	\end{align}
	Let $\theta_1 = \frac{3/2{\eta}\tau^2L^2}{1-4\tau^2{\eta}^2L^2}\leq\frac{1}{2}\tau L$ and $\eta_0=\eta_m=\eta\leq\frac{1}{4(\tau+1)L}$ and  choose $\theta_2,\cdots,\theta_\tau$ as
	\begin{align}\label{eq:proof-9}
		\theta_{i+1} = \theta_i - {\frac{3}{2}{\eta}L^2}\tau - 4\tau\theta_1L^2{\eta}^2,\quad \text{for}\ i=1, \cdots,\tau-1
	\end{align}
	Following form Eq.~xx and the definition of $\theta_1$, there is $\theta_{\tau} = \theta_1 - (\tau-1)\frac{3{\eta}L^2}{2}\tau - 4(\tau-1)\tau\theta_1L^2{\eta}^2\geq 0$. Then Eq.~\ref{eq:proof-8} reduces to
	\begin{align}\label{eq:proof-10}
		&\BE (M^{t+1} - M^t)
		\leq -\frac{1}{2}\min_{m}p_m\eta\BE\|\nabla f_{\mu_m}(w_0,\BW)\|^2 + 2{L\eta^2} \sigma_0^2 \nonumber \\
		&+ {\sum_{m=1}^{q} p_m\eta^2}(L+2\tau L)\sigma_m^2 +
		+ \sum_{m=1}^{q}p_m(\frac{\eta}{4}+\frac{3L\eta^2}{2}+ \frac{3}{2}\tau L\eta^2 ){\mu_m}^2L_m^2  (d_m+3)^3
		+(\frac{\eta}{4} +\frac{3 L\eta^2}{2}){\mu_m}^2L_0^2 (d_0+3)^3
	\end{align}
	Summing Eq.~\ref{eq:proof-10} over $t=0,\cdots,T-1$, there is
	\begin{align}\label{eq:proof-11}
		\frac{1}{T}\sum_{t=0}^{T-1}\BE\|\nabla f_{\mu_m}(w_0,\BW)\|^2
		&\leq \frac{f_{\mu_m}^0-f_{\mu_m}^*}{\frac{1}{2}\min_{m}p_mT\eta}
		+ \frac{{\sum_{m=1}^{q} p_m\eta}(L+2\tau L)\sigma_m^2 + {2L\eta} \sigma_0^2}{\frac{1}{2}\min_{m}p_m}
		\nonumber \\
		& + \frac{\sum_{m=1}^{q}p_m(\frac{1}{4}+\frac{3L\eta+}{2} \frac{3}{2}\tau L\eta ){\mu_m}^2L_m^2  (d_m+3)^3
			+(\frac{1}{4} + \frac{3L\eta}{2}){\mu_m}^2L_0^2 (d_0+3)^3}{\frac{1}{2}\min_{m}p_m}
	\end{align}
	According to Lemma~\ref{lemma: smooth_f_random}, there is
	\begin{align}\label{eq:proof-12}
		\BE\|\nabla_{m} f(w_0,\BW)\|^2 \leq 2\BE\|\nabla_m f_{\mu_m}(w_0,\BW)\|^2 + \frac{{\mu_m}^2L_m^2(d_m+3)^3}{2}.
	\end{align}
	Thus, there is
	\begin{align}\label{eq:proof-13}
		\BE\|\nabla f(w_0,\BW)\|^2 &\leq 2\sum_{m=0}^{q}\BE\|\nabla_m f_{\mu_m}(w_0,\BW)\|^2 + \sum_{m=0}^{q} \frac{{\mu_m}^2L_m^2(d_m+3)^3}{2}
		\nonumber \\
		& \leq 2\BE\|\nabla f_{\mu_m}(w_0,\BW)\|^2 + \sum_{m=0}^{q} \frac{{\mu_m}^2L_m^2(d_m+3)^3}{2}.
	\end{align}
	Similarly, according to Lemma~\ref{lemma: smooth_f_random}, there is
	\begin{align}\label{eq:proof-14}
		f(w_0^0,\BW^0) - f^* \leq f_{\mu_m}(w_0^0,\BW^0)-f_{\mu_m}^* + \sum_{m=0}^{q}\frac{{L_md_m{\mu_m}^2}}{2}
	\end{align}
	Applying Eqs. \ref{eq:proof-13} and \ref{eq:proof-14} to Eq.~\ref{eq:proof-11}, there is
	\begin{align}\label{eq:proof-15}
		\frac{1}{T}\sum_{t=0}^{T-1}\BE\|\nabla f(w_0,\BW)\|^2
		&\leq \frac{f^0-f^*}{\frac{1}{4}\min_{m}p_mT\eta}
		+ \frac{{\sum_{m=1}^{q} p_m\eta}(L+2\tau L)\sigma_m^2 + {2L\eta} \sigma_0^2}{\frac{1}{4}\min_{m}p_m}
		+ \sum_{m=0}^{q}\frac{{L_md_m{\mu_m}^2}}{2T}
		\nonumber \\
		& + \sum_{m=0}^{q} \frac{{\mu_m}^2L_m^2(d_m+3)^3}{2} + \frac{\sum_{m=0}^{q}p_m(\frac{1}{4}+\frac{3L\eta}{2}+ \frac{3}{2}\tau L\eta ){\mu_m}^2L_m^2 (d_m+3)^3
		}{\frac{1}{4}\min_{m}p_m}
	\end{align}
	Let $L_*=\max\{\{L_m\}_{m=0}^q,L\}$, $d_* = \max\{d_m+3\}_{m=0}^{q}$, $\sigma_*^2= \max_{m}\sigma_m^2$, $\frac{1}{p_*}=\min_mp_m$, then Eq.~xx reduces to
	\begin{align}\label{eq:proof-16}
		\frac{1}{T}\sum_{t=0}^{T-1}\BE\|\nabla f(w_0,\BW)\|^2
		&\leq \frac{4p_*(f^0-f^*)}{T\eta}
		+ {8p_*(L+\tau L){\eta}\sigma_*^2}
		+\frac{{(q+1)L_*d_*{\mu_m}^2}}{2T}
		+  \frac{(q+1){\mu_m}^2L_*^2d_*^3}{2}
		\nonumber \\
		& + {p_*(2+3L_*\eta+ \frac{3}{2}\tau L_*\eta ){\mu_m}^2L_*^2 d_*^3}
	\end{align}
	Choosing $\eta=\min\{\frac{1}{4(\tau+1)L},\frac{m_0}{\sqrt{T}}\}$ with constant $m_0>0$ and ${\mu_m}=\mathcal{O}(\frac{1}{\sqrt{T}})$ such as ${\mu_m} = \frac{1}{\sqrt{T}L_*d_*^{3/2}}$, there is
	\begin{align}\label{eq:proof-17}
		\frac{1}{T}\sum_{t=0}^{T-1}\BE\|\nabla f(w_0,\BW)\|^2
		&\leq \frac{4p_*(f^0-f^*)}{\sqrt{T}m_0}
		+ \frac{{8{p_*}m_0(L+\tau L)\sigma_*^2}}{\sqrt{T}}
		+\frac{{(q+1)}}{2T^2L_*d_*^2}
		+  \frac{(q+1)+3p_*}{2T}
	\end{align}
	Thus, if $\tau$ is a constant independent to $T$, then there
\end{proof}

\newpage
\onecolumn
\begin{lemma}
	\label{lemma: smooth_f_random}
	Suppose that Assumption~\ref{assum1} holds, then  we have
	
	1) $f_{\mu_m}(w_m)$ is $L_m$-smooth and $f_{\mu_m}(w_0,\BW)$ is $L$-smooth
	\begin{align}\label{eq: smooth_est_grad}
		\nabla_m f_{\mu_m} (w_0,\BW)
		=  \mathbb E_{u } \left [  \hat \nabla_{m} f(w_0,\BW) \right ]
		\\
		\nabla f_{\mu_m} (w_0,\BW)
		=  \mathbb E_{u }
		\left [
		\hat \nabla f(w_0,\BW)
		\right ]
	\end{align}
	where $\mathbf u$ is drawn from
	the uniform distribution over the unit Euclidean
	sphere, and
	$\hat \nabla_{w_m} f(w_0,\BW)$ is given by Eq.~(6).
	
	2) For any $w_m \in \mathbb R^{d_{m}}$,
	\begin{align}
		& | f_{\mu_m}( w_m) - f(w_m) | \leq \frac{L_md_m {\mu_m^2}}{2} \label{eq: dist_f_smooth_true_random}  \\
		& \| \nabla_m f_{\mu_m} (w_0,\BW) - \nabla_{m} f(w_0,\BW) \|_2^2 \leq \frac{{\mu_m}^2 L_m^2 d_{m}^2}{4}, \label{eq: dist_smooth_true}
	\end{align}
\end{lemma}
\begin{lemma}
	Under Assumptions~\ref{assum0} to \ref{assum3}, for $m=0,1,\cdots,q$ there is
	\begin{align}\label{eq:proof-0}
		\BE\|\whv_\mt^t-v_\mt^t\|^2 &=  \BE\|\whv_\mt^t - \bv_\mt^t + \bv_\mt^t -v_\mt^t\|^2
		\nonumber \\
		&\leq  2\BE\|\whv_\mt^t - \bv_\mt^t\| + 2\BE\|\bv_\mt^t -v_\mt^t\|^2
		\nonumber \\
		&\leq  \frac{{\mu_m}^2L_\mt^2d_\mt^2}{2} + 2L^2\|\BWW^t -\BW^t\|^2
	\end{align}
	and
	\begin{align}\label{eq:proof-1}
		\BE\|\whv_\mt^t\|^2 &=  \BE\|\whv_\mt^t - v_\mt^t + v_\mt^t\|^2
		\nonumber \\
		&\leq  2\BE\|\whv_\mt^t - v_\mt^t\| + 2\BE\|v_\mt^t\|^2
		\nonumber \\
		&\leq  {{\mu_m}^2L_\mt^2d_\mt^2} + 4L^2\|\BWW^t -\BW^t\|^2 + 2\BE\|v_\mt^t\|^2
		\nonumber \\
		&\leq  3{{\mu_m}^2L_\mt^2d_\mt^2} + 4L^2\|\BWW^t -\BW^t\|^2 + 4\sigma_\mt^2
	\end{align}
\end{lemma}
\begin{proof}
	Under Assumption~\ref{assum1},  and taking expectation w.r.t. the sample index $i_t$, we have
	\begin{align}\label{eq:proof-2}
		&\BE f_{\mu_m}(w_0^{t+1},\BW^{t+1})
		\nonumber \\
		&= \BE \left( f_{\mu_m}(w_0^{t}-\eta_0\whv_0^t
		,\cdots, w_{m_t}^t-\eta_{m_t}\whv_{m_t}^t,\cdots) \right)
		\nonumber\\
		&\leq  \BE \left(  f_{\mu_m}(w_0^{t},\BW^{t}) -\eta_0 \left<V_0^t,\whv_0^t\right> - \eta_{m_t}\left<V_\mt^t,\whv_\mt^t\right>
		+ \frac{L\eta_0^2}{2}\|\whv_0^t\|^2
		+  \frac{L\eta_{m_t}^2}{2}\|\whv_{m_t}^t\|^2 \right)
		\nonumber\\
		&= \BE \left(  f_{\mu_m}(w_0^{t},\BW^{t}) -\eta_0 \left<V_0^t,\whv_0^t - v_0^t + v_0^t\right>
		- \eta_{m_t}\left<V_m^t,\whv_\mt^t -  v_\mt^t + v_\mt^t\right>
		+ \frac{L\eta_0^2}{2}\|\whv_0^t\|^2
		+  \frac{L\eta_{m_t}^2}{2}\|\whv_{m_t}^t\|^2 \right)
		\nonumber\\
		&\leq \BE \left(  f_{\mu_m}(w_0^{t},\BW^{t}) -\eta_0\|V_0^t\|^2 - \eta_0 \left<V_0^t,\whv_0^t - v_0^t \right>
		- \eta_{m_t} \|V_\mt^t\|^2 - \eta_{m_t} \left<V_\mt^t,\whv_\mt^t -  v_\mt^t\right>
		+ \frac{L\eta_0^2}{2}\|\whv_0^t\|^2
		+  \frac{L\eta_{m_t}^2}{2}\|\whv_{m_t}^t\|^2 \right)
		\nonumber\\
		&\leq \BE \left(  f_{\mu_m}(w_0^{t},\BW^{t}) -\frac{\eta_0}{2}\|V_0^t\|^2 + \frac{\eta_0}{2}\|\whv_0^t - v_0^t\|^2
		- \frac{\eta_{m_t}}{2} \|V_\mt^t\|^2 + \frac{\eta_{m_t}}{2}\|\whv_\mt^t -  v_\mt^t\|
		+ \frac{L\eta_0^2}{2}\|\whv_0^t\|^2
		+  \frac{L\eta_{m_t}^2}{2}\|\whv_{m_t}^t\|^2 \right)
		\nonumber\\
		&\leq \BE \left(  f_{\mu_m}(w_0^{t},\BW^{t}) -\frac{\eta_0}{2}\|V_0^t\|^2
		- \frac{\eta_{m_t}}{2} \|V_\mt^t\|^2
		+ (\eta_0 + \eta_\mt + 2L\eta_0^2 + 2L\eta_\mt^2) L^2\|\BWW^t -  \BW^t\|^2\right)
		\nonumber\\
		& + (\frac{\eta_\mt}{4} + \frac{3L\eta_\mt^2}{2}){{\mu_m}^2L_\mt^2d_\mt^2}
		+ (\frac{\eta_0}{4} + \frac{3L\eta_0^2}{2}){{\mu_m}^2L_0^2d_0^2} + 2L\eta_\mt^2\sigma_\mt^2 + 2L\eta_0^2\sigma_0^2
	\end{align}
	Taking expectation w.r.t. $m_t$, and using Assumption~\ref{assum2}, there is
	\begin{align}\label{eq:proof-3}
		&\BE f_{\mu_m}(w_0^{t+1},\BW^{t+1})\nonumber \\
		&\leq \BE  f_{\mu_m}(w_0^{t},\BW^{t}) -\frac{\eta_0}{2}\BE\|V_0^t\|^2
		-\sum_{m=1}^{q} p_m \frac{\eta_{m}}{2}\BE \|V_m^t\|^2
		+ \underbrace{(\eta_0 + 2L\eta_0^2 + \max_{m}(2L\eta_m^2+\eta_m)) L^2}_{\beta^t} \BE\|\BWW^t -  \BW^t\|^2
		\nonumber\\
		& +{\sum_{m=1}^{q} p_m(\frac{\eta_m}{4} + \frac{3L\eta_m^2}{2}){\mu_m}^2L_m^2}  d_\mt^2
		+ (\frac{\eta_0}{4} +\frac{3 L\eta_0^2}{2}){\mu_m}^2L_0^2d_0^2
		+ {\sum_{m=1}^{q} p_m2L\eta_m^2}\sigma_m^2 + 2{L\eta_0^2} \sigma_0^2
	\end{align}
	According to Assumption~\ref{assum3}, there is
	\begin{align}\label{eq:proof-4}
		\|\BWW^t -  \BW^t\|^2 = \|\sum_{i \in D(t)}\BW^{i+1}-\BW^{i}\|^2 \leq \tau\sum_{i=1}^{\tau}\|\BW^{t+1-i}-\BW^{t-i}\|^2
	\end{align}
	We than bound the term $ \|\BWW^t -  \BW^t\|^2$. First, for $\BE\|\BW^{t+1} -  \BW^t\|^2$
	\begin{align}\label{eq:proof-5}
		\BE \|\BW^{t+1} -  \BW^t\|^2  = \BE\eta_\mt^2\|\whv_\mt^2 \|
		\leq \sum_{m=1}^{q}p_m \eta_m^2(3{{\mu_m}^2L_m^2d_\mt^2} + 4\sigma_m^2) + \max_{m}\eta_m^2 4L^2\|\BWW^t -\BW^t\|^2
	\end{align}
	Define a Lyapunov function as
	\begin{align}\label{eq:proof-6}
		M^t = f_{\mu_m}(w_0^t,\BW^t) + \sum_{i=1}^{\tau}\theta_i\|\BW^{i+1}-\BW^{i}\|^2
	\end{align}
	Following Lemma~\ref{lemma: smooth_f_random} and Eq.~\ref{eq:proof-6}, there is
	\begin{align}\label{eq:proof-7}
		&\BE (M^{t+1} - M^t) \nonumber \\
		&= \BE\left( f_{\mu_m}(w_0^{t+1},\BW^{t+1}) + \sum_{i=1}^{\tau}\theta_i\|\BW^{t+1+1-i}-\BW^{t+1-i}\|^2
		- f_{\mu_m}(w_0^t,\BW^t) - \sum_{i=1}^{\tau}\theta_i\|\BW^{t+1-i}-\BW^{t-i}\|^2 \right)
		\nonumber \\
		& =  -\frac{\eta_0}{2}\BE\|V_0^t\|^2
		-\sum_{m=1}^{q} p_m \frac{\eta_{m}}{2}\BE \|V_m^t\|^2
		+ \beta^t\BE\|\BWW^t -  \BW^t\|^2
		+ {\sum_{m=1}^{q} p_m2L\eta_m^2}\sigma_m^2 + 2{L\eta_0^2} \sigma_0^2
		\nonumber \\
		& +{\sum_{m=1}^{q} p_m(\frac{\eta_m}{4} + \frac{3L\eta_m^2}{2}){\mu_m}^2L_m^2}  d_{m}^2
		+(\frac{\eta_0}{4} + \frac{3L\eta_0^2}{2}){\mu_m}^2L_0^2 d_{0}^2
		\nonumber \\
		& + \theta_1\BE \|\BW^{t+1}-\BW^t\|^2 + \sum_{i=1}^{\tau-1}(\theta_{i+1}-\theta_i)\BE \|\BW^{t+1-i}-\BW^{t-i}\|^2  - \theta_\tau\BE \|\BW^{t+1-\tau}-\BW^{t-\tau}\|^2
		\nonumber \\
		& \leq
		-\frac{\eta_0}{2}\BE\|V_0^t\|^2
		-\sum_{m=1}^{q} p_m \frac{\eta_{m}}{2}\BE \|V_m^t\|^2
		+ \beta^t\tau\sum_{i=1}^{\tau}\|\BW^{t+1-i}-\BW^{t-i}\|^2
		+ {\sum_{m=1}^{q} p_m2L\eta_m^2}\sigma_m^2 + 2{L\eta_0^2} \sigma_0^2
		\nonumber \\
		& +{\sum_{m=1}^{q} p_m(\frac{\eta_m}{4} + \frac{3L\eta_m^2}{2}){\mu_m}^2L_m^2}  d_{m}^2
		+ (\frac{\eta_0}{4} + \frac{3L\eta_0^2}{2}){\mu_m}^2L_0^2 d_{0}^2
		\nonumber \\
		&
		+ \sum_{i=1}^{\tau-1}(\theta_{i+1}-\theta_i)\BE \|\BW^{t+1-i}-\BW^{t-i}\|^2  - \theta_\tau\BE \|\BW^{t+1-\tau}-\BW^{t-\tau}\|^2
		\nonumber \\
		& + \theta_1(\sum_{m=1}^{q}p_m \eta_m^2(3{{\mu_m}^2L_m^2d_m^2} + 4\sigma_m^2) + \max_{m}\eta_m^2 4L^2\tau\sum_{i=1}^{\tau}\|\BW^{t+1-i}-\BW^{t-i}\|^2)
		\nonumber \\
		& \leq
		-\frac{\eta_0}{2}\BE\|V_0^t\|^2
		-\sum_{m=1}^{q} p_m \frac{\eta_{m}}{2}\BE \|V_m^t\|^2
		+ {\sum_{m=1}^{q} p_m\eta_m^2}(L+4\theta_1)\sigma_m^2
		+ 2{L\eta_0^2} \sigma_0^2
		\nonumber \\
		& + \sum_{m=1}^{q}p_m(\frac{\eta_m}{4}+\frac{3L\eta_m^2}{2}+ 3\theta_1\eta_m^2 ){\mu_m}^2L_m^2  d_{m}^2
		+(\frac{\eta_0}{4} + \frac{3L\eta_0^2}{2}){\mu_m}^2L_0^2 d_{0}^2
		\nonumber \\
		&
		+ \sum_{i=1}^{\tau-1}(\beta_t\tau + \tau \theta_1\max_{m}\eta_m^2 4L^2 + \theta_{i+1}-\theta_i)\BE \|\BW^{t+1-i}-\BW^{t-i}\|^2  + (\beta_t\tau +  \tau \theta_1\max_{m}\eta_m^2 4L^2 -\theta_\tau)\BE \|\BW^{t+1-\tau}-\BW^{t-\tau}\|^2
	\end{align}
	If we choose $\eta_0,\eta_m\leq \bar{\eta}\leq\frac{1}{4(L+2\theta_1)}$, then there is $\beta^t\leq\frac{3\bar{\eta}L^2}{2}$. Then for Eq.~\ref{eq:proof-7} there is
	\begin{align}\label{eq:proof-8}
		&\BE (M^{t+1} - M^t) \nonumber \\
		& \leq -\frac{1}{2}\min\{\eta_0,p_m\eta_m\}\BE\|\nabla f_{\mu_m}(w_0,\BW)\|^2
		+ {\sum_{m=1}^{q} p_m\eta_m^2}(L+4\theta_1)\sigma_m^2
		+ 2{L\eta_0^2} \sigma_0^2
		\nonumber \\
		&  + \sum_{m=1}^{q}p_m(\frac{\eta_m}{4}+\frac{3L\eta_m^2}{2}+ 3\theta_1\eta_m^2 ){\mu_m}^2L_m^2  d_{m}^2
		+(\frac{\eta_0}{4} + \frac{3L\eta_0^2}{2}){\mu_m}^2L_0^2 d_{0}^2
		\nonumber \\
		&
		-\sum_{i=1}^{\tau-1}(\theta_i - \theta_{i+1} - {\frac{3}{2}\bar{\eta}L^2}\tau - 4\tau\theta_1L^2\bar{\eta}^2) \BE \|\BW^{t+1-i}-\BW^{t-i}\|^2
		- (\theta_\tau - {\frac{3}{2}\bar{\eta}L^2}\tau - 4\tau\theta_1L^2\bar{\eta}^2)\BE \|\BW^{t+1-\tau}-\BW^{t-\tau}\|^2
	\end{align}
	Let $\theta_1 = \frac{3/2{\eta}\tau^2L^2}{1-4\tau^2{\eta}^2L^2}\leq\frac{1}{2}\tau L$ and $\eta_0=\eta_m=\eta\leq\frac{1}{4(\tau+1)L}$ and  choose $\theta_2,\cdots,\theta_\tau$ as
	\begin{align}\label{eq:proof-9}
		\theta_{i+1} = \theta_i - {\frac{3}{2}{\eta}L^2}\tau - 4\tau\theta_1L^2{\eta}^2,\quad \text{for}\ i=1, \cdots,\tau-1
	\end{align}
	Following form Eq.~\ref{eq:proof-6} and the definition of $\theta_1$, there is $\theta_{\tau} = \theta_1 - (\tau-1)\frac{3{\eta}L^2}{2}\tau - 4(\tau-1)\tau\theta_1L^2{\eta}^2\geq 0$. Then Eq.~\ref{eq:proof-8} reduces to
	\begin{align}\label{eq:proof-10}
		&\BE (M^{t+1} - M^t)  \leq -\frac{1}{2}\min_{m}p_m\eta\BE\|\nabla f_{\mu_m}(w_0,\BW)\|^2 + 2{L\eta^2} \sigma_0^2
		\nonumber \\
		&+ {\sum_{m=1}^{q} p_m\eta^2}(L+2\tau L)\sigma_m^2
		+ \sum_{m=1}^{q}p_m(\frac{\eta}{4}+\frac{3L\eta^2}{2}+ \frac{3}{2}\tau L\eta^2 ){\mu_m}^2L_m^2  d_{m}^2
		+(\frac{\eta}{4} +\frac{3 L\eta^2}{2}){\mu_m}^2L_0^2 d_{0}^2
	\end{align}
	Summing Eq.~ over $t=0,\cdots,T-1$, there is
	\begin{align}\label{eq:proof-11}
		\frac{1}{T}\sum_{t=0}^{T-1}\BE\|\nabla f_{\mu_m}(w_0,\BW)\|^2
		&\leq \frac{f_{\mu_m}^0-f_{\mu_m}^*}{\frac{1}{2}\min_{m}p_mT\eta}
		+ \frac{{\sum_{m=1}^{q} p_m\eta}(L+2\tau L)\sigma_m^2 + {2L\eta} \sigma_0^2}{\frac{1}{2}\min_{m}p_m}
		\nonumber \\
		& + \frac{\sum_{m=1}^{q}p_m(\frac{1}{4}+\frac{3L\eta+}{2} \frac{3}{2}\tau L\eta ){\mu_m}^2L_m^2  d_{m}^2
			+(\frac{1}{4} + \frac{3L\eta}{2}){\mu_m}^2L_0^2 d_{0}^2}{\frac{1}{2}\min_{m}p_m}
	\end{align}
	According to Lemma~\ref{lemma: smooth_f_random}, there is
	\begin{align}\label{eq:proof-12}
		\BE\|\nabla_{m} f(w_0,\BW)\|^2 \leq 2\BE\|\nabla_m f_{\mu_m}(w_0,\BW)\|^2 + \frac{{\mu_m}^2L_m^2d_{m}^2}{2}.
	\end{align}
	Thus, there is
	\begin{align}\label{eq:proof-13}
		\BE\|\nabla f(w_0,\BW)\|^2 &\leq 2\sum_{m=0}^{q}\BE\|\nabla_m f_{\mu_m}(w_0,\BW)\|^2 + \sum_{m=0}^{q} \frac{{\mu_m}^2L_m^2d_{m}^2}{2}
		\nonumber \\
		& \leq 2\BE\|\nabla f_{\mu_m}(w_0,\BW)\|^2 + \sum_{m=0}^{q} \frac{{\mu_m}^2L_m^2d_{m}^2}{2}.
	\end{align}
	Similarly, according to Lemma~\ref{lemma: smooth_f_random}, there is
	\begin{align}\label{eq:proof-14}
		f(w_0^0,\BW^0) - f^* \leq f_{\mu_m}(w_0^0,\BW^0)-f_{\mu_m}^* + \sum_{m=0}^{q}\frac{{L_m{\mu_m}^2}}{2}
	\end{align}
	Applying Eqs. \ref{eq:proof-13} and \ref{eq:proof-14} to Eq.~\ref{eq:proof-11}, there is
	\begin{align}\label{eq:proof-15}
		\frac{1}{T}\sum_{t=0}^{T-1}\BE\|\nabla f(w_0,\BW)\|^2
		&\leq \frac{f^0-f^*}{\frac{1}{4}\min_{m}p_mT\eta}
		+ \frac{{\sum_{m=1}^{q} p_m\eta}(L+2\tau L)\sigma_m^2 + {2L\eta} \sigma_0^2}{\frac{1}{4}\min_{m}p_m}
		+ \sum_{m=0}^{q}\frac{{L_md_m{\mu_m}^2}}{2T}
		+ \sum_{m=0}^{q} \frac{{\mu_m}^2L_m^2d_{m}^2}{2}
		\nonumber \\
		& + \frac{\sum_{m=0}^{q}p_m(\frac{1}{4}+\frac{3L\eta}{2}+ \frac{3}{2}\tau L\eta ){\mu_m}^2L_m^2 d_{m}^2
		}{\frac{1}{4}\min_{m}p_m}
	\end{align}
	Let $L_*=\max\{\{L_m\}_{m=0}^q,L\}$, $d_* = \max\{d_m\}_{m=0}^{q}$, $\sigma_*^2= \max_{m}\sigma_m^2$, $\frac{1}{p_*}=\min_mp_m$, then Eq.~xx reduces to
	\begin{align}\label{eq:proof-16}
		\frac{1}{T}\sum_{t=0}^{T-1}\BE\|\nabla f(w_0,\BW)\|^2
		&\leq \frac{4p_*(f^0-f^*)}{T\eta}
		+ {8p_*(L+\tau L){\eta}\sigma_*^2}
		+\frac{{(q+1)L_*{\mu_m}^2}}{2T}
		+  \frac{(q+1){\mu_m}^2L_*^2d_*^2}{2}
		\nonumber \\
		& + {p_*(2+3L_*\eta+ \frac{3}{2}\tau L_*\eta ){\mu_m}^2L_*^2 d_*^2}
	\end{align}
	Choosing $\eta=\min\{\frac{1}{4(\tau+1)L},\frac{m_0}{\sqrt{T}}\}$ with constant $m_0>0$ and ${\mu_m}=\mathcal{O}(\frac{1}{\sqrt{T}})$ such as ${\mu_m} = \frac{1}{\sqrt{T}L_*d_*}$, there is
	\begin{align}\label{eq:proof-17}
		\frac{1}{T}\sum_{t=0}^{T-1}\BE\|\nabla f(w_0,\BW)\|^2
		&\leq \frac{4p_*(f^0-f^*)}{\sqrt{T}m_0}
		+ \frac{{8{p_*}m_0(L+\tau L)\sigma_*^2}}{\sqrt{T}}
		+\frac{{(q+1)}}{2T^2L_*d_*^2}
		+  \frac{(q+1)+3p_*}{2T}
	\end{align}
	Thus, if $\tau$ is a constant independent to $T$, we can drive the corresponding result.
\end{proof}

%
\newpage
\bibliographystyle{ACM-Reference-Format}
{\bibliography{ref}}

\end{document}